\documentclass[lettersize,11pt]{extarticle}
\usepackage{amsmath, amsfonts, amsthm, amssymb, bbm}
\usepackage{algorithmic}
\usepackage{algorithm}
\usepackage{array}
\usepackage[margin=1in]{geometry}
\usepackage{fancyhdr}
\usepackage{setspace}
\usepackage{textcomp}
\usepackage{stfloats}
\usepackage{url}
\usepackage{verbatim}
\usepackage{graphicx}
\usepackage{multirow}
\usepackage{booktabs}
\usepackage{cite}
\usepackage{subcaption}
\usepackage{xcolor}

\usepackage{hyperref}

\usepackage{mathtools}
\mathtoolsset{showonlyrefs}
\allowdisplaybreaks
\usepackage{enumerate}

\newcommand{\bs}[1]{\boldsymbol{#1}}
\newcommand{\R}{\mathbb{R}}

\newcommand{\diag}[1]{\mathrm{diag}\left({#1}\right)}

\newcommand{\W}{\boldsymbol{W}}
\newcommand{\w}{\boldsymbol{w}}

\newcommand{\balpha}{\boldsymbol{\alpha}}
\newcommand{\bbeta}{\boldsymbol{\beta}}

\renewcommand{\v}{\boldsymbol{v}}
\newcommand{\p}{\boldsymbol{p}}
\newcommand{\q}{\boldsymbol{q}}
\newcommand{\x}{\boldsymbol{x}}
\newcommand{\y}{\boldsymbol{y}}
\newcommand{\z}{\boldsymbol{z}}
\newcommand{\btheta}{\boldsymbol{\theta}}

\renewcommand{\H}{\boldsymbol{H}}
\newcommand{\J}{\boldsymbol{J}}
\renewcommand{\b}{\boldsymbol{b}}

\newtheorem{remark}{Remark}
\newcommand{\gn}{\texttt{GradNet}}
\newcommand{\gns}{\texttt{GradNets}}
\newcommand{\mgn}{\texttt{mGradNet}}
\newcommand{\mgns}{\texttt{mGradNets}}
\newcommand{\cmgn}{\texttt{mGradNet-C}}
\newcommand{\mmgn}{\texttt{mGradNet-M}}

\newcommand{\mmgns}{\texttt{mGradNet-Ms}}
\newcommand{\gnc}{\texttt{GradNet-C}}

\newcommand{\gnm}{\texttt{GradNet-M}}
\newcommand{\gnms}{\texttt{GradNet-M}s}
\newcommand{\lse}{\mathrm{LSE}}
\newcommand{\softmax}{\mathrm{softmax}}
\renewcommand{\hat}[1]{\widehat{#1}}

\newtheorem{definition}{Definition}
\newtheorem{lemma}{Lemma}
\newtheorem{theorem}{Theorem}
\newtheorem{corollary}{Corollary}[theorem] 
\newtheorem{prop}{Proposition}[theorem]

\newcommand\blfootnote[1]{
    \begingroup
    \renewcommand\thefootnote{}\footnote{#1}
    \addtocounter{footnote}{-1}
    \endgroup
}

\fancypagestyle{plain}{
    \fancyhf{} 
    \fancyhead[C]{Published in IEEE Transactions on Signal Processing vol. 73, pp. 324-339, 2025 \href{https://doi.org/10.1109/TSP.2024.3496692}{10.1109/TSP.2024.3496692}}
    \fancyfoot[L]{
    \footnotesize{
    \textcopyright\ 2025 IEEE.  Personal use of this material is permitted.  Permission from IEEE must be obtained for all other uses, in any current or future media, including reprinting/republishing this material for advertising or promotional purposes, creating new collective works, for resale or redistribution to servers or lists, or reuse of any copyrighted component of this work in other works.
    }
    }
}

\begin{document}
\title{Gradient Networks}

\author{Shreyas Chaudhari\footnotemark[1] \and Srinivasa Pranav\footnotemark[1] \and Jos\'e M.F. Moura
\blfootnote{Equal contribution\\
Authors partially supported by NSF Graduate Research Fellowships (DGE-1745016, DGE-2140739), NSF Grant CCF-2327905, and ARCS Fellowship.\\ Code available at \href{https://github.com/SPronav/GradientNetworks}{https://github.com/SPronav/GradientNetworks}}
}

\date{
Carnegie Mellon University\\
Electrical and Computer Engineering}

\maketitle

\begin{abstract}
Directly parameterizing and learning gradients of functions has widespread significance,
with specific applications in
inverse problems, generative modeling, and optimal transport. This paper introduces gradient networks (\gns): novel neural network architectures
that parameterize gradients of various function classes.
\gns\ exhibit specialized architectural constraints that ensure correspondence to gradient functions. We provide a comprehensive \gn\ design framework that includes methods for transforming \gns\ into monotone gradient networks (\mgns), which are guaranteed to represent gradients of convex functions. Our results establish that our proposed \gn\ (and \mgn) universally approximate the gradients of (convex) functions.
Furthermore, these networks can be customized to correspond to specific spaces of potential functions, including transformed sums of (convex) ridge functions.
Our analysis leads to two distinct \gn\ architectures, \gnc\ and \gnm, and we describe the
corresponding monotone versions, \cmgn\ and \mmgn. Our empirical results demonstrate that these architectures provide efficient parameterizations and outperform existing methods by up to 15 dB in gradient field tasks and by up to 11 dB in Hamiltonian dynamics learning tasks.
\end{abstract}


\section{Introduction} 
\label{sec:intro}

Deep neural networks are prized for their ability to parameterize and easily learn complicated, high-dimensional functions. Researchers have devoted substantial effort to developing deep neural networks that achieve state-of-the-art performance on numerous tasks spanning computer vision~\cite{krizhevsky2012imagenet}, natural language processing~\cite{vaswani2017attention}, and reinforcement learning~\cite{mnih2015human}. These neural networks are commonly unconstrained and effectively parameterize the space of all functions.
Many applications instead require learned functions that exhibit specific properties, which necessitates the design of neural networks corresponding to specific function classes -- a problem that has seldom been studied. Constraining neural networks to belong to a particular function class not only enhances interpretability, but also leads to theoretical performance guarantees essential for deploying trained models in safety-critical applications.

In particular, neural networks corresponding to \textit{gradients} of functions hold significant importance across several science and engineering disciplines. For example, physicists often wish to use a finite set of measurements to characterize the gradient field of a potential function, such as the temperature gradient over a surface. Score-based generative models are another application of learning gradient functions, where neural networks are trained to learn $\nabla_{\x} \log p(\x)$, the score function of an unknown probability distribution~\cite{hyvarinen2005estimation, song2019generative}. These methods have been applied to image \cite{songDiff, song2020improved} and 3D shape~\cite{cai2020learning} generation.

Learning gradients of \textit{convex} functions holds particular significance in optimal transport theory. Brenier's theorem states that the unique solution to the Monge problem with Euclidean cost is given by the gradient of a convex function~\cite{brenier1991polar, santambrogio2015optimal}. The theorem has inspired methods for learning Monge maps \cite{korotin2021neural, bunne2022supervised, makkuva2020optimal, Our_ICASSP_Paper} using gradients of parameterized convex functions.
Applications of learned gradients of convex functions also extend to gradient-based optimization. Optimization routines can incorporate a learned gradient function to define a set of iterative updates that map an input to a desired output~\cite{DOvongkulbhisal2018discriminative}.
{In particular, gradients of convex functions can be used to define gradients of regularization terms in optimization objectives, as in the regularization by denoising (RED) \cite{romano2017little, reehorst2018regularization} and plug-and-play (PnP) \cite{venkatakrishnan2013plug, kamilov2023plug} frameworks for solving inverse problems. In order to enhance interpretability and obtain convergence guarantees, recent works have designed denoisers for these frameworks that have symmetric Jacobians \cite{cohen2021has, hurault2022, gavaskar2021plug}, thereby ensuring that the denoisers are gradient functions.} 


Previously in~\cite{Our_ICASSP_Paper}, we considered learning gradients of convex functions and proposed monotone gradient networks. In this paper, we generalize the problem setting in~\cite{Our_ICASSP_Paper} to consider significantly broader function classes. We introduce new \gn\ architectures for learning gradients of these function classes, provide extensive theoretical analysis concerning \gn\ universal approximation capabilities, and validate our methods with extensive experiments.

\textbf{Contributions}: We propose gradient networks (\gns), neural networks for directly parameterizing gradients of arbitrary functions. We adapt \gns\ to monotone gradient networks (\mgns) that correspond to gradients of convex functions. Our proposed neural network architectures directly model the gradient function $\nabla F$ without first learning the underlying potential function $F$. We also analyze the expressivity of our proposed \gn\ and \mgn\ networks and prove that they are universal approximators of gradients of general and convex functions, respectively. We further describe methods for modifying these networks to represent popular subsets of these function classes,
including
gradients of sums of (convex) ridge functions and their (convexity-preserving) transformations. Our theoretical results translate into improved experimental results over existing works. In our gradient field approximation experiments, we find that our architectures achieve up to 15 dB lower mean squared error (MSE) than existing methods. Additionally, we predict the dynamics of the two-body problem and demonstrate an 11 dB improvement over existing methods.

\textbf{Paper Organization}: We first review existing methods for parameterizing and learning gradients of functions in Sec.~\ref{sec:related_work}. We then present \textit{gradient networks} (\gns) for learning gradients of arbitrary functions in Sec.~\ref{sec:gradient_networks}. In Sec.~\ref{sec:monotone_gradient_networks}, we describe \textit{monotone gradient networks} (\mgns): modified \gns\ for representing gradients of {convex} functions. We analyze the expressivity of our proposed \gns\ and \mgns\ in Sec.~\ref{sec:universal_approximation} and demonstrate methods for composing \gns\ and \mgns\ to learn (sub)gradients of of the function classes described in Fig.~\ref{fig:venn_diagram}. We present architectural modifications in Sec.~\ref{sec:enhancements} that empirically yield efficient parameterizations. Finally, in Sec.~\ref{sec:experiments}, we evaluate our proposed models on gradient field and Hamiltonian dynamics learning tasks.  

\begin{figure}[ht]
    \centering
    \includegraphics[width=0.8\columnwidth]{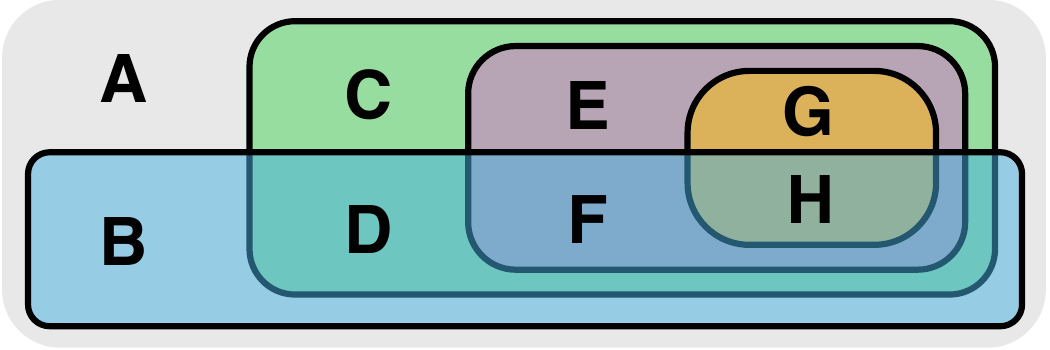}
    \caption{Relationships between relevant function classes considered in corresponding theory sections of this paper: A) all functions; B) monotone functions; Sec.~\ref{sec:monotone_gradients}: C) gradients, D) monotone gradients; Sec.~\ref{sec:transformations}: E) gradients of transformed sums of ridges, F) gradients of transformed sums of \textit{convex} ridges; Sec.~\ref{sec:ridges}: G) gradients of sums of ridges, H) gradients of sums of \textit{convex} ridges; Sec.~\ref{sec:subgradients}: respective subgradients.}
    \label{fig:venn_diagram}
\end{figure}

\textbf{Notation}: 
$\bs{w}$ and $\bs{W}$ respectively denote vectors and matrices. $\nabla F(\x)$ is the gradient of a potential function $F$ at a point $\x$ and is a vector-valued function. $\H_{F}(\x)$ is the Hessian of $F$ at $\x$, the matrix of second-order partial derivatives. When a function $f$ is vector-valued, $\J_{f}(\x)$ denotes the Jacobian of $f$ at $\x$, the matrix of partial derivatives. $\bs{A} \succeq 0$ indicates that $\bs{A}$ is symmetric positive semidefinite (PSD). $\mathrm{diag}(\cdot)$ is the vector-to-diagonal-matrix operator. $C^0(\mathcal{D})$ and $C^k(\mathcal{D})$ respectively denote the space of continuous and $k$-times continuously differentiable functions that map from domain $\mathcal{D}$ to the real numbers $\R$. $C^k_\times(\mathcal{D})$ are the \textit{convex} functions in $C^k(\mathcal{D})$. 


\section{Related Work}
\label{sec:related_work}

Several existing works use standard neural networks to directly model the gradient of a potential function~\cite{fermanian2023pnp, song2020improved, andrychowicz2016learning, reehorst2018regularization}. While these methods demonstrate satisfactory empirical performance, they lack theoretical justification and are often considered heuristic approaches. It is not guaranteed, and in fact highly unlikely, that an arbitrary neural network with matching input and output dimension corresponds to the gradient of a scalar-valued function. More formally, let $\mathcal{F}$ be a set such that each function $F \in \mathcal{F}$ is the gradient of a scalar-valued function. Given an $F \in \mathcal{F}$, standard neural network architectures, e.g., multilayer perceptrons (MLPs) and convolutional neural networks (CNNs), with nonpolynomial activations, can be used to approximate $F$. However, these architectures offer no guarantee that the learned function $\hat{F}$ is itself in $\mathcal{F}$. In contrast, the methods and network architectures discussed in this paper ensure that the approximator $\hat{F}$ is always a member of $\mathcal{F}$, thereby enhancing interpretability and enabling robust theoretical guarantees in practice.



Rather than directly approximating the gradient of a function using a standard neural network, an alternative approach may be to first parameterize and learn the underlying scalar potential function using a standard neural network. Subsequently differentiating the network with respect to its input yields the desired gradient. However, closely approximating a target function $F$ with an {arbitrary approximator} $G$ does not guarantee that $\nabla G$ closely approximates $\nabla F$ \cite{rudin1964principles}. Empirical evidence in \cite{czarnecki2017sobolev} confirms this fact when the approximating functions are neural networks. Other works directly train the gradient map of a neural network that parameterizes a potential function. For instance,~\cite{hurault2022} considers the potential $F(\x) = \frac{1}{2}\|\x - f_{\btheta}(\x)\|_2^2$ with neural network $f_{\btheta} : \R^n \to \R^n$ and uses automatic differentiation to obtain a gradient network architecture $\nabla F(\x)$. Nevertheless, these approaches often exhibit ill-behaved product structures in practice \cite{GradientNetworksSingleFeature} and, due to the Runge phenomena~\cite{metz2021gradients}, optimizing their parameters becomes cumbersome as the input dimension grows~\cite{ICGN}. 



{Feedforward neural networks have also been proposed to directly parameterize gradients of specific classes of potential functions, namely potentials $F(\x)$ expressible as the sum of ridge functions: $F(\x) = \sum_{i=1}^n \psi_i(\w_i^\top \x)$. Such potentials were considered in the \textit{fields of experts} framework proposed in \cite{FieldsOfExperts}. The fields of experts framework has inspired several methods for learning image priors that effectively generalize the ``transform domain thresholding'' approach discussed in \cite{reehorst2018regularization, donoho1994ideal}. These include works like \cite{patchbasedsparsemodelstohigherorderMRFs, hammernik2018learning} that correspond to nonconvex potentials and respectively achieve commendable performance for image restoration and magnetic resonance imaging reconstruction.
More recently,~\cite{EPFL, goujon2024learning} propose neural networks with elementwise, learnable spline activations to learn gradients of potentials expressible as sums of ridge functions (class G in Fig.~\ref{fig:venn_diagram}). The networks can be modified to correspond to gradients of sums of \textit{convex} ridge functions (class H in Fig.~\ref{fig:venn_diagram}) by restricting the splines to be nondecreasing. In this paper, we consider parameterizing the activations as learnable neural networks and discuss how these networks can be adapted to represent (sub)gradients of convex potentials. Furthermore, the class of functions expressible as a sum of convex ridges, as in~\cite{EPFL, goujon2024learning}, does not include several common convex functions including
\begin{align}
    \text{Infinity Norm}: &\max \{|x_1|, |x_2|, \dots, |x_n|\}\\
    \text{Exponential Product}: &\exp\left(\sum_{i=1}^n x_i \right)
\end{align} Thus, we also use group activations to yield universal approximation of gradients of (convex) function classes that extend beyond those expressible as sums of ridges.
}

The literature on parameterizing and learning the gradients of convex functions (monotone gradient functions) includes various approaches. Input Convex Neural Networks (ICNN)~\cite{ICNN} parameterize only convex functions by requiring positive weight matrices and convex, monotonically increasing elementwise activation functions. To extract the gradient, ~\cite{chen2018optimal, makkuva2020optimal, ConvexPotentialFlows} use a two-step approach of first evaluating a convex potential parameterized by an ICNN and then differentiating the network with respect to its input (via backpropagation). While~\cite{ConvexPotentialFlows} provides theoretical justification for this approach, the ICNN's architectural restrictions result in well-documented training difficulties in practice~\cite{ICGN, korotin2021neural, sivaprasad2021curious, hoedt2024principled}. In this paper, we avoid these challenges and bypass the two-step process by directly characterizing and learning the gradient.

{Input Convex Gradient Networks (ICGNs)~\cite{ICGN} extend the approach in \cite{lorraine2019jacnet} and characterize the gradient of twice-differentiable convex functions by exploiting the symmetric positive semidefinite structure of their Hessians. ICGNs specifically parameterize a Gram factor of the Hessian. The gradient map is then obtained by numerically integrating the learned Hessian. However, as noted in \cite{ICGN}, the only known architecture suitable for parameterizing the Gram factor is $\sigma(\W\x + \b)$, for which numerical integration is not required and serves only as a proof of concept. {These architectures, for which numerical integration is not required, parameterize gradients of sums of convex ridges (class H in Fig.~\ref{fig:venn_diagram}). Meanwhile, in this work we consider parameterizing functions for several other larger classes of functions, including transformations of sums of (convex) ridges, differentiable convex functions, and $L$-smooth nonconvex functions.} Additionally, the existence of deeper networks suitable for parameterizing the Gram factor of the ICGN is conjectured in \cite{ICGN}, but have not been identified to date. Numerical integration for deeper networks would also introduce computational challenges in high dimensions (e.g., error accumulation, convergence)~\cite{numerical_integration}.}  


\section{Gradient Networks (\gn)}
\label{sec:gradient_networks}

In this section, we introduce \textit{gradient networks} (\gns) for parameterizing and learning gradients of continuously differentiable functions. We first state sufficient conditions for a neural network to be a \gn\ and then discuss various approaches for constructing \gns.
\begin{definition}(\gn)
    A gradient network ({\gn}) is a neural network $f$ satisfying $f = \nabla F$ for some $F \in C^1\left(\R^d\right)$.
\label{def:gn}
\end{definition}
By Def.~\ref{def:gn}, differentiating a neural network that parameterizes a function $F \in C^1(\R^d)$ yields a \gn. However, due to the discussion in Sec.~\ref{sec:related_work}, we focus on directly modeling and learning $\nabla F$ (without first learning $F$). We now consider methods for parameterizing \gns. In the case where a \gn\ parameterizes the gradient of a \textit{twice} continuously differentiable function, {we use the following known result}.
{{\begin{lemma}[Antiderivatives and Symmetric Jacobians]
     A differentiable function $f : \R^d\to \R^d$ has a scalar-valued antiderivative if and only if its Jacobian is symmetric everywhere, i.e., $\forall \x \in \R^d\;\J_{f}(\x) = \J_{f}(\x)^\top$.
     \label{lemma:gn_jac_sym}
 \end{lemma}

The sufficient condition in Lemma~\ref{lemma:gn_jac_sym} follows from the theorem of symmetric second derivatives credited to Clairaut, Schwarz, Young, and others \cite{tao_analysis2}. The necessary condition follows from a slight modification of Prop.~1 in~\cite{ICGN} to consider all functions with symmetric Jacobians.}
Thus, a neural network is a \gn\ if its Jacobian with respect to its input is everywhere symmetric. We use Lemma~\ref{lemma:gn_jac_sym} to provide a neural \gn\ parameterization in Prop.~\ref{prop:single_layer_gn} below.
}
\begin{prop}
    The neural network
    \begin{equation}
    \bs{W}^{\top}\sigma(\bs{W}\x + \bs{a}) + \bs{b}
        \label{eq:single_layer_gn}
    \end{equation}
   is a \gn\ if there exists $\psi \in C^1(\R^m)$ such that $\sigma = \nabla \psi$. In particular, such a $\psi$ exists if activation $\sigma:\R^m \to \R^m$ is differentiable and its Jacobian $\J_{\sigma}$ is everywhere symmetric.
   \label{prop:single_layer_gn}
\end{prop}
\begin{proof}
\eqref{eq:single_layer_gn} is the gradient of $\psi(\bs{W}\x + \bs{a}) + \bs{b}^\top\x$. If $\sigma$ is additionally differentiable with $\J_{\sigma}$ everywhere symmetric then, by Lemma~\ref{lemma:gn_jac_sym}, there exists $\psi \in C^2(\R^d)$ such that $\sigma = \nabla \psi$.
\end{proof}

Prop.~\ref{prop:single_layer_gn} states that~\eqref{eq:single_layer_gn} is a \gn\ if the activation function $\sigma$ has an antiderivative. It hence permits the use of \textit{group} activations such as softmax. Furthermore, to guarantee existence of an antiderivative, it is sufficient, but not necessary, that $\sigma$ is differentiable and its Jacobian is symmetric everywhere. For example, the elementwise activation $\mathrm{ReLU}(\x) = \max(0, \x)$ is continuous and nondifferentiable, but it is the gradient of $\sum_{i}\max(0, \frac{1}{2}x_i^2)$.
{Architectures of the form in \eqref{eq:single_layer_gn}, with a continuous elementwise activation function $\sigma(\x) = \begin{bmatrix}
    \sigma_1(x_1)\ \dots\ \sigma_m(x_m) 
\end{bmatrix}^\top$, include ``transform domain thresholding'' methods \cite{reehorst2018regularization, donoho1994ideal} where the activation is a continuous thresholding operator. {Similar architectures appear in \cite{goujon2024learning}, which uses linear spline activations, and in \cite{hammernik2018learning}, which parameterizes each $\sigma_i$ as a linear combination of Gaussian radial basis functions}. An alternative approach is to use neural networks to parameterize the elementwise activation functions. 
\begin{remark}
    \eqref{eq:single_layer_gn} is a \gn\ if $\sigma(\x)$ is an elementwise activation where each $\sigma_i$ is a neural network in $C^1(\R)$.
    \label{rem:gn_nn_activations}
\end{remark}}

{ We later prove in Sec.~\ref{sec:ridges} that the network in Rem.~\ref{rem:gn_nn_activations} can approximate the gradient of any differentiable function that is the sum of ridge functions. In addition, if we restrict the domain of the \gn\ to be compact, then it is sufficient for~\eqref{eq:single_layer_gn} to have continuous elementwise activation $\sigma$, since all continuous $\sigma_i$ are integrable on compact subsets of $\R$.} Lastly, by linearity of the gradient operator, the networks in \eqref{eq:single_layer_gn}, and other \gn\ parameterizations, can be linearly combined to yield another \gn.
{\begin{remark}
    Linear combinations of \gns\ are also \gns.
    \label{rem:gn_linear_combinations}
\end{remark}
}

This section established methods for designing \gns: neural networks guaranteed to correspond to gradients of continuously differentiable functions. It introduced a specific \gn\ architecture, which we analyze in Sec.~\ref{sec:universal_approximation}, and serves as a foundation for designing other \gn\ architectures like \textit{monotone} gradient networks in the following section. 


\section{Monotone Gradient Networks (\mgn)}
\label{sec:monotone_gradient_networks}


Monotone gradient networks (\mgns) are neural networks guaranteed to represent gradients of convex functions and are a subclass of gradient networks (\gn) \cite{Our_ICASSP_Paper}. This section introduces \mgns\ in a manner similar to the presentation of \gns\ in Sec.~\ref{sec:gradient_networks}. We first define \mgns\ and then discuss examples and their properties. 

\begin{definition}(\mgn)
    A monotone gradient network ({\mgn}) is a neural network $f$ satisfying $f = \nabla F$ for some convex $F \in C_{\times}^1\left(\R^d\right)$.
    \label{def:mgn}
\end{definition} 

In Sec.~\ref{sec:subgradients}, we extend Def.~\ref{def:mgn} to accommodate subgradients of non-differentiable convex functions. Next, we recall that a differentiable function is convex if and only if its gradient is monotone \cite{rockafellar2009variational}. Therefore, \mgns\ are guaranteed to be monotone functions. 
\begin{definition}[Monotonicity]
    \label{def:monotone}
    $f : \R^d \to \R^d$ is monotone if:
    \begin{align}
        \forall\x,\y \in \R^d,\;(f(\x) - f(\y))^\top(\x - \y) \geq 0
        \label{eq:monotone}
    \end{align}
\end{definition}
{
It is generally challenging to design an \mgn\ satisfying~\eqref{eq:monotone} for all possible pairs of inputs. Instead, it is more tractable to rely on single-input characterizations of monotonicity.
First, we provide an approach to parameterize gradients of convex functions using \gns\ (from Sec.~\ref{sec:gradient_networks}) and known Lipschitz regularity techniques.
\begin{remark}
    Let $L>0$ and $f$ be a \gn\ that is $L$-Lipschitz, i.e., $\forall \x,\y\in \R^d,\; \|f(\x)-f(\y)\|\leq L \|\x-\y\|$. Then $g(\x) = L\x - f(\x)$ is an \mgn.
    \label{rem:gn_as_mgn}
\end{remark}
{Rem.~\ref{rem:gn_as_mgn} follows from Def.~\ref{def:gn} and Remark~2.2 in~\cite{goujon2024learning}, which implies monotonicity of $g$. It demonstrates how to construct an \mgn\ using a \gn\ with known Lipschitz constant. Rem.~\ref{rem:gn_as_mgn} relates to several Regularization by Denoising (RED) works that design learnable functions of the form $\x - f(\x)$ where $f(\x)$ is constrained to be nonexpansive~\cite{ryu2019plug, terris2020building}. These methods differ from our approach as they do not guarantee that the Jacobian of $\x - f(\x)$ is symmetric.
}

Next, we propose an alternative method that directly parameterizes monotone gradients and avoids Lipschitz assumptions.}
We use the fact that $F \in C^2(\R^d)$ is convex if and only if its Hessian $\H_f$ is everywhere positive semidefinite (PSD)~\cite{boyd2004convex}. Since $\bs{H}_F = \bs{J}_{\nabla F}^\top$, using a differentiable \mgn\ to parameterize a monotone $\nabla F$ requires the Jacobian of the \mgn, with respect to its input, to be PSD everywhere.
{\begin{prop}
    The neural network in~\eqref{eq:single_layer_gn} is an \mgn\ if there exists convex $\psi \in C_{\times}^1(\R^m)$ such that $\sigma = \nabla \psi$. In particular, such a $\psi$ exists if activation $\sigma:\R^m \to \R^m$ is differentiable and its Jacobian $\J_{\sigma}$ is everywhere PSD.
    \label{prop:single_layer_mgn}
\end{prop}
\begin{proof}
    \eqref{eq:single_layer_gn} is the gradient of $\psi(\bs{W}\x + \bs{a}) + \bs{b}^\top\x$, which is convex since $\psi$ is convex. If $\sigma$ is additionally differentiable with $\J_{\sigma} \succeq 0$, then by Lemma~\ref{lemma:gn_jac_sym} and the fact that a twice differentiable function is convex if and only if its Hessian is PSD~\cite{boyd2004convex}, there exists $\psi \in C^2(\R^d)$ such that $\sigma = \nabla \psi$.
\end{proof}}
We highlight three key points concerning Prop.~\ref{prop:single_layer_mgn}: 1) A neural network with Jacobian that is everywhere PSD is guaranteed to be an \mgn; specifically, a \gn\ in Prop.~\ref{prop:single_layer_gn} with elementwise, monotone activation $\sigma$ is an \mgn. 2) Most popular elementwise activations, including softplus\footnote{The softplus function $\frac{1}{\beta}\log(1 + \exp(\beta x))$, with scaling factor $\beta$, is a smooth approximation of the commonly used ReLU activation.}, tanh, and sigmoid, are nondecreasing and can be used to specify $\sigma$ in Prop.~\ref{prop:single_layer_mgn}. 3) Similar to the discussion on \gns\ in Sec.~\ref{sec:gradient_networks}: if the activation $\sigma$ in Prop.~\ref{prop:single_layer_mgn} has an antiderivative, it need not be differentiable; if we consider a compact domain, continuity of an elementwise activation $\sigma$ is sufficient.

{We leverage the generality of Prop.~\ref{prop:single_layer_mgn} to propose \mgns\ with group activation $\sigma$ and show that they universally approximate gradients of all differentiable convex functions. These \mgns\ with group activations are provably more expressive than the methods in \cite{FieldsOfExperts,donoho1994ideal,patchbasedsparsemodelstohigherorderMRFs,hammernik2018learning, EPFL,goujon2024learning}. Specifically, in Sec.~\ref{sec:monotone_gradients}, we prove that there exist sequences of \mgns\ with softmax activation functions that can universally approximate gradients of differentiable convex functions. The generality of Prop.~\ref{prop:single_layer_mgn} also permits neural parameterizations of elementwise activations}.
{\begin{remark}
    \eqref{eq:single_layer_gn} is an \mgn\ if $\sigma(\x)$ is an elementwise activation where each $\sigma_i$ is an \mgn\ in $C^1(\R)$.
    \label{rem:mgn_mgn_activations}
\end{remark}
{The \mgns\ specified by Rem.~\ref{rem:mgn_mgn_activations} are more amenable for direct implementation than the spline activations proposed in \cite{EPFL, goujon2024learning}, while being provably as expressive. In particular, the spline activations in \cite{EPFL, goujon2024learning} are defined using a fixed grid and hence require careful tuning of the grid domain and a large number of knots to achieve satisfactory performance. In contrast, we observe in Sec.~\ref{sec:experiments} that \mgns\ in Rem.~\ref{rem:mgn_mgn_activations} provide more efficient parameterizations of the activations.

Similar to techniques described in \cite{mukherjee2020learned, goujon2024learning}, \mgns\ can be modified to correspond to gradients of $\mu$-strongly convex functions.}
\begin{remark}
    Let $\mu > 0$ and $f$ be an \mgn. The function $g(\x) = f(\x) + \mu\x$ is an \mgn\ corresponding to the gradient of a $\mu$-strongly convex function.
    \label{rem:strongly_monotone}
\end{remark}
As gradients of strongly convex functions are invertible, the method in Rem.~\ref{rem:strongly_monotone} parameterizes \textit{invertible} \mgns, which, for example, can be employed in normalizing flows \cite{ConvexPotentialFlows}. If an \mgn\ corresponds to the gradient of a strongly convex function $F$, then its inverse corresponds to the gradient of the Fenchel dual $F^*$ \cite{rockafellar1970convex}.
Lastly, we note that conical combinations of convex functions are also convex~\cite{boyd2004convex}. Analogous to Rem.~\ref{rem:gn_linear_combinations}, \mgns\ can be combined to produce other \mgns.
\begin{remark}
Conical combinations (linear combinations with nonnegative coefficients) of \mgns\ yield \mgns. \label{rem:mgn_conical_combinations}
\end{remark}


\section{Universal Approximation Results}
\label{sec:universal_approximation}
In this section, we analyze the expressivity of \gn\ and \mgn\ architectures respectively specified by Propositions~\ref{prop:single_layer_gn} and~\ref{prop:single_layer_mgn}.
We show that these architectures can \textit{universally approximate} various function classes. We first formally define universal approximation and a set of gradient functions, which we use throughout the section. 
\begin{definition}[Universal Approximation]
    Let $\mathcal{S}\subset\R^d$ be compact, $\mathcal{F} : \mathcal{S} \to \R^d$ be a class of continuous functions, and $\mathcal{G} : \mathcal{S} \to \R^d$ be a class of approximators. $\mathcal{G}$ universally approximates $\mathcal{F}$ if for any $f\in \mathcal{F}$, there exists a sequence of $g_n \in \mathcal{G}$ that uniformly converges to $f$.
    \label{def:universal_approx}
\end{definition}
Def.~\ref{def:universal_approx} is equivalent to $\mathcal{G}$ being dense in $\mathcal{F}$ with respect to the supremum norm.
Since $\mathcal{S}$ must be a subset of some scaled and shifted version of $[0,1]^d$, our universal approximation proofs, without loss of generality, consider the domain $[0,1]^d$.
\begin{definition}[Set of Gradient Functions]
    Let $\mathcal{F}$ be a set of differentiable functions. Then the set $\nabla \mathcal{F} = \{\nabla F : F \in \mathcal{F}\}$. 
\end{definition}
 In Sec.~\ref{sec:monotone_gradients}, we prove that the \mgn\ in~\eqref{eq:single_layer_gn} with scaled softmax activation and increasing hidden dimension can universally approximate the gradient of any convex function. We extend this result to show that the difference of two \mgns\ can universally approximate the gradient of any $L$-smooth function.

In Sec.~\ref{sec:ridges}, we analyze the impact of the activation function on the approximation capabilities of the networks. We show that \gns\ and \mgns\ with nonpolynomial activations can learn the gradient of any function expressible as the sum of ridge and convex ridge functions, respectively. In Sec.~\ref{sec:transformations}, we introduce a simple augmentation that enables our networks to universally approximate the gradient of a transformed sum of (convex) ridges. {Finally, in Sec.~\ref{sec:subgradients}, we extend \mgn\ results to subgradients of convex functions}.

\subsection{(Monotone) Gradient Functions}
\label{sec:monotone_gradients}
We start by proving that \mgns\ of the form~\eqref{eq:single_layer_gn} can universally approximate monotone gradients of convex functions. The proof uses the following lemma: differentiating {convex} approximators of a convex potential yields {monotone} approximators of the monotone gradient of the potential. 
\begin{lemma}[Convex Function and Monotone Gradient Approximation: \cite{rockafellar1970convex} Theorem 25.7]
Let $\mathcal{S} \subseteq \R^d$ be open, convex and let $F\in C^1_\times(\mathcal{S})$ be finite. Let $G_i \in C^1_\times(\mathcal{S})$ be a sequence of finite functions such that $\forall \x \in \mathcal{S},~\lim_{i \to \infty} G_i(\x) = F(\x)$. Then $\forall \x \in \mathcal{S},\lim_{i \to \infty} \nabla G_i(\x) = \nabla F(\x)$. The sequence $\nabla G_i$ converges uniformly to $\nabla F$ on every compact subset of $\mathcal{S}$.
\label{lemma:rockafeller}
\end{lemma}
Lemma~\ref{lemma:rockafeller} is specific to convex functions and generally does not extend to arbitrary functions~\cite{rudin1964principles}. We prove that \mgns\ can universally approximate monotone gradients by first constructing a sequence of convex functions that approximates any continuous convex function, and then applying Lemma~\ref{lemma:rockafeller}. To construct the sequence, we use the class of scaled LogSumExp (LSE) functions with scaling factor $t > 0$:
\begin{equation}
    \mathrm{LSE}_t(\mathcal{U}) = \frac{1}{t}\log\left(\sum_{u \in \mathcal{U}} \exp(tu)\right)
    \label{eq:lse}
\end{equation}
{We first present a result of independent interest in Lemma~\ref{lemma:lse_approx} below, where we derive an upper bound on the approximation error incurred when using the scaled LSE of affine functions to approximate continuous convex functions. 

\begin{lemma}[Bound on Convex Function Approximation by LSE]
    Let convex $F \in C^{0}_{\times}\left([0,1]^d\right)$ and $\epsilon > 0$. There exist scaling factor $t > 0$ and parameters $\{\w_{i}, b_i\}_{i=1}^n$, where $n = (2^m -1)^d$ and $m>0$ depend on $F$, such that the scaled LSE of affine functions $G(\x)=\mathrm{LSE}_t(\{\w_1^\top \x + b_1, \dots, \w_n^\top \x + b_n\})$ satisfies
    \begin{equation}
        \sup_{\x \in [0,1]^d} |F(\x) - G(\x)| < (d+1)\epsilon + \frac{\log n}{t} 
        \label{eqn:lse_error_bound}
    \end{equation}
    \label{lemma:lse_approx}
\end{lemma}
\begin{proof}
    Let $\epsilon > 0$ and $F\in C_{\times}^0([0,1]^d)$.
    By the Heine-Cantor Theorem~\cite{Heine_Cantor}, $F$ is uniformly continuous on $[0,1]^d$, meaning $\exists \delta > 0$ such that $\forall \x,\y \in [0,1]^d$, $\|\x-\y\| < \delta \implies |F(\x)-F(\y)| < \epsilon$.
    We select $m \in \mathbb{N}$ such that $2^{-m} < \delta$ and define $\mathcal{X}$ as the set of points with coordinates lying in $\{i2^{-m} : 1 \leq i \leq 2^m -1\}$. This means $n = |\mathcal{X}| = (2^m -1)^d$.
    For each $\y \in \mathcal{X}$, let $L_{\y}(\x) = \v^\top (\x - \y) + F(\y)$ be a supporting hyperplane of $F$ at $\y$,
    where $\v$ is a subgradient of $F$ at $\y$, i.e., $\v^\top (\x - \y) \leq F(\x) - F(\y)$. 
With $\x\in[0,1]^d$, $t>0$, and  $\mathcal{U} = \{L_{\y}(\x): \y \in \mathcal{X}\}$, the  triangle inequality implies
\begin{align}
    \left| F(\x) - \lse_t\left(\mathcal{U}\right) \right|
    &\leq \left| F(\x) - \max_{y\in\mathcal{X}} L_{\y}(\x)\right| + \left|\max_{y\in\mathcal{X}} L_{\y}(\x) - \lse_t\left(\mathcal{U}\right) \right| \label{eqn:max_triangle} \\
    &< (d+1)\epsilon + \frac{d}{t}\log(2^m-1)\\
    &= (d+1)\epsilon + \frac{\log n}{t} \label{eqn:max_bound}
\end{align}
The first term of~\eqref{eqn:max_bound} uses Appendix C Prop.~2 of~\cite{ConvexPotentialFlows} and the second term uses
$\lse_t(\mathcal{U}) < \max(\mathcal{U}) + \frac{1}{t}\log(|\mathcal{U}|)$~\cite{calafiore_LSE}.
\end{proof} 
In Lemma~\ref{lemma:lse_approx}, choosing small $\epsilon$ and large scaling factor $t$ corresponds to closely approximating a convex function $F$ with a scaled LSE of affine functions. In the bound~\eqref{eqn:lse_error_bound}, $n$ is the number of affine functions and is equivalent to the number of hidden neurons in the neural networks we propose and analyze in Thm.~\ref{thm:universal_approximation}. The number of neurons $n$ depends on $\epsilon$ and uniform continuity properties of $F$.

To illustrate the utility of the error bound, we consider the case where the convex function $F$ is $L$-Lipschitz on $[0,1]^d$. This encompasses a wide range of functions, including mean squared error and polynomial functions. The variables $\delta$ and $\epsilon$ in the proof of Lemma~\ref{lemma:lse_approx} can then be related by taking $\delta < \epsilon / L$, leading to the following bound on the number of neurons $n$: $n > \left(\frac{L}{\epsilon} - 1\right)^d$. The bound on $n$ indicates the rate at which the number of neurons must increase when either the input dimension increases, the Lipschitz constant of $F$ increases, or we demand closer approximations ($\epsilon$ decreases). To achieve a desired approximation error, the scaling factor $t$ should be increased logarithmically with respect to $n$.
If we wish to approximate $F$ with the sum of $k$ identical $\mathrm{LSE}_t$ functions, then each $\mathrm{LSE}_t$ function must approximate $\frac{1}{k} F$, which is $\frac{L}{k}$-Lipschitz and $n > \left(\frac{L}{k\epsilon} - 1\right)^d$. Approximating $F$ as a sum of identical functions, each being a scaled LSE of affine functions, thus allows each individual $LSE_t$ to use fewer affine functions. This observation motivates our design of \mmgns\ in Sec.~\ref{sec:mmgn}.

}

{Before leveraging Lemma~\ref{lemma:rockafeller} to show universal approximation results for \mgns, we first include the following result from \cite{calafiore_LSE}, which states that scaled LSE functions universally approximate convex functions.
\begin{lemma}[Scaled LSE as Universal Approximator of Convex Functions: \cite{calafiore_LSE} Theorem 2]
Let $\x \in \R^d$ and $\mathcal{G}$ be scaled LSE of affine functions $\mathrm{LSE}_t(\{\w_1^\top \x + b_1, \dots, \w_n^\top \x + b_n\})$ with scaling factor $t > 0$. $\mathcal{G}$ universally approximates $C_{\times}^0([0,1]^d)$. 
\label{lemma:lse_dense}
\end{lemma}
We defer the proof of Lemma~\ref{lemma:lse_dense} to \cite{calafiore_LSE}.\footnote{A different approach to the proof in \cite{calafiore_LSE} easily follows from Lemma~\ref{lemma:lse_approx}.}
}
Combining Lemma~\ref{lemma:rockafeller} with Lemma~\ref{lemma:lse_dense} enables us to show in Thm.~\ref{thm:universal_approximation} below that the \mgn\ in Prop.~\ref{prop:single_layer_mgn} can universally approximate gradients of convex functions. The proof is constructive and demonstrates that the sequence of approximating functions corresponds to a sequence of \mgns\ with softmax activation and increasing hidden dimension.

\begin{theorem}[Universal Gradient Approximation for Convex Functions]
\label{thm:universal_approximation}
Let $\mathcal{F} = C^1_\times ([0-\delta,1+\delta]^d)$ with $\delta>0$. \mgns\ in Prop.~\ref{prop:single_layer_mgn} with scaled softmax activation universally approximate $\nabla \mathcal{F}$ on $[0,1]^d$.
\end{theorem}
\begin{proof}
Let $F \in \mathcal{F}$ and let $G_n \in C^2_\times ([0-\delta,1+\delta]^d)$ be a sequence of the form $\mathrm{LSE}_t(\{\w_1^\top \x + b_1, \dots, \w_n^\top \x + b_n\})$ that converges uniformly to $F$ by Lemma~\ref{lemma:lse_dense}. By the extreme value theorem~\cite{rudin1964principles}, $F$ and all $G_n$ are finite.
We consider the open convex subset $\mathcal{A} = (0-\delta,1+\delta)^d \subset [0-\delta,1+\delta]^d$ and observe that the $G_n$ also converge uniformly to $F$ on $\mathcal{A}$. By Lemma~\ref{lemma:rockafeller}, $\nabla G_n \to \nabla F$ uniformly on compact subsets of $\mathcal{A}$, including $[0,1]^d$.
Note that $\forall n,\; \nabla G_n$ is an \mgn\ in Prop.~\ref{prop:single_layer_mgn} with scaled softmax activation and the $i$th row of $\W$ being $\w_i^\top$: 
$\nabla G_n = \W^\top \softmax(t(\W\x+\b))$
\end{proof}
Using the fact that \mgns\ can universally approximate all monotone gradient functions, we show that the difference of two \mgns\ with softmax activations can universally approximate the gradient of any $L$-smooth function. 
\begin{theorem}[Universal Gradient Approximation for L-smooth Functions]
   Let $\delta > 0$ and $\mathcal{F}$ be $L$-smooth functions in $C^1([0-\delta,1+\delta]^d)$. Let $g_1(\x)$ and $g_2(\x)$ be \mgns\ in Prop.~\ref{prop:single_layer_mgn} with scaled softmax activations. \gns\ $g(\x) = g_1(\x) - g_2(\x)$ universally approximate $\nabla \mathcal{F}$ on $[0,1]^d$.
   \label{thm:gn_universal_approx}
\end{theorem}

\begin{proof}
    {Let $F \in \mathcal{F}$, implying $\nabla F$ is $L$-Lipschitz continuous. We can write $\nabla F = L\x - (L\x - \nabla F(\x))$, where $L\x$ is clearly monotone and $L\x - \nabla F(\x)$ is monotone by Rem.~\ref{rem:gn_as_mgn}}. By Thm.~\ref{thm:universal_approximation}, there exist sequences of \mgns\ uniformly converging to $L\x$ and $L\x - \nabla F(\x)$ on $[0,1]^d$.
    By Rem.~\ref{rem:gn_linear_combinations}, the difference of these sequences is a sequence of \gns\ that converges uniformly to $\nabla F$.
\end{proof}


\subsection{Gradients of Sums of (Convex) Ridge Functions}
\label{sec:ridges}


{Elementwise activations $\sigma(\x) = \begin{bmatrix}
\sigma_1(x_1)\ \dots\ \sigma_n(x_n)  
\end{bmatrix}^\top$ are highly parallellizable in practice and are more commonly used than group activation functions like softmax. Prior works~\cite{FieldsOfExperts,donoho1994ideal,patchbasedsparsemodelstohigherorderMRFs,hammernik2018learning, EPFL,goujon2024learning}, demonstrated the empirical success of elementwise activations used in \gns\ of the form in~\eqref{eq:single_layer_gn}. While we proved universal approximation for \textit{any} (convex) gradient in Sec.~\ref{sec:universal_approximation}, this section demonstrates that using an elementwise activation function in~\eqref{eq:single_layer_gn} compromises representation power. Specifically, it leads to learning gradients of sums of (convex) ridge functions.}

\begin{definition}[Sum of (Convex) Ridge Functions]
    $F \in C^k(\R^d)$ is expressible as the finite sum of ridge functions if 
    \begin{equation}
        F(\x) = \sum_{i=1}^N \psi_i(\bs{a}_i ^\top \x + b_i)
        \label{eq:ridges}
    \end{equation}
    where each profile function $\psi_i \in C^k(\R)$. $F$ is expressible as the sum of convex ridge functions if each $\psi_i \in C_{\times}^k(\R)$.
    \label{def:ridge_function}
\end{definition}
Thm.~\ref{thm:gn_ridges_universal} below shows that \gns\ with scaled elementwise activations (e.g., scaled sigmoid, tanh, ReLU) can learn the gradient of any function that is the sum of ridges. {It extends the result in \cite{EPFL}, which considers approximating gradients of sums of convex ridges using learnable linear spline activations.} 


{\begin{theorem}[Universal Gradient Approximation for Sums of Ridges]
    Let $\mathcal{F}$ be functions in $C^1([0,1]^d)$ expressible as a finite sum of ridge functions. \gns\ of the form in~\eqref{eq:single_layer_gn} with continuous, scaled, elementwise nonpolynomial activation $\sigma$ universally approximate $\nabla \mathcal{F}$.
    \label{thm:gn_ridges_universal}
\end{theorem}}
\begin{proof}
    If $F \in \mathcal{F}$, then $\nabla_{\x} F(\x) = \sum_{i=1}^N \bs{a}_i \psi_i'(\bs{a}_i^\top \x + b_i)$. Let $\bs{A} = \begin{bmatrix}
        \bs{a}_1\ \dots\ \bs{a}_N
    \end{bmatrix}^\top$, $\bs{b} = \begin{bmatrix}
        b_1\ \dots\ b_N
    \end{bmatrix}^\top$, and $\Psi(\cdot) = \begin{bmatrix}
        \psi'_1(\cdot)\ \dots\ \psi'_N(\cdot) 
    \end{bmatrix}^\top$. Then $\nabla_{\x} F(\x) = \bs{A}^\top\Psi(\bs{A}\x + \bs{b})$, where each $\bs{a}_i^\top(\cdot) + b_i$ is a continuous affine transformation that maps $[0,1]^d$ to a compact subset $\mathcal{S}_i\subset \R$. We now consider the \gn\ in Prop.~\ref{prop:single_layer_gn} with $N \times d$ matrix $\W=\bs{A}$ and continuous elementwise activation $\sigma(\x) = \begin{bmatrix}
    \sigma_1(x_1)\ \dots\ \sigma_n(x_n)  
    \end{bmatrix}^\top$. For each continuous $\psi_i'$,
    let $\epsilon_i > 0$
    and consider a corresponding $\sigma_i$ satisfying
    $\forall x \in \mathcal{S}_i,~|\psi_i'(x) - \sigma_i(x)| < \epsilon_i$.
    This \gn's approximation error is bounded as follows: 
    \begin{align}
        \left\|\sum_{i=1}^N \bs{a}_i (\psi_i'(\bs{a}_i^\top \x + b_i) - \sigma_i(\bs{a}_i^\top \x + b_i))\right\|
        &\leq \sum_{i=1}^N |\psi_i'(\bs{a}_i^\top \x + b_i) - \sigma_i(\bs{a}_i^\top \x + b_i)|\left\|\bs{a}_i \right\|\\
        &\leq \sum_{i=1}^N \epsilon_i\|\bs{a}_i \|
    \end{align}
    Given any error threshold $\epsilon > 0$, there exist sufficiently small $\epsilon_i$, and corresponding $\sigma_i$, such that the RHS of the final inequality above is bounded by $\epsilon$.

    Next, we parameterize each $\sigma_i$ as a neural network of the form
    $\sigma_i(x) = \bs{u}_i^\top s(\bs{v}_i x + \bs{\beta}_i),\; \bs{u}_i, \bs{v}_i, \bs{\beta}_i\in\R^{m_i}$
    with continuous, elementwise nonpolynomial activation $s$ (e.g., tanh, sigmoid). By Thm. 1 in~\cite{Nonpolynomial_LESHNO1993861}, for each $\psi_i'$, there exists $\sigma_i$ satisfying $\forall x \in \mathcal{S}_i,~|\psi_i'(x) - \sigma_i(x)| < \epsilon_i$.
    Substituting the specified form of $\sigma_i$ into the aforementioned \gn\ yields
    $\sum_{i=1}^N \bs{a}_i \bs{u}_i^\top s(\bs{v}_i \cdot (\bs{a}_i^\top \x + b_i) + \bs{\beta}_i)$. This
    can be rewritten as
    $\sum_{i=1}^N \bs{A}^\top_i \mathrm{diag}(\bs{u}_i)s(\mathrm{diag}(\bs{v}_i) (\bs{A}_i \x + \bs{b}_i) + \bs{\beta}_i)$,
    where the $m_i \times d$ matrix $\bs{A}_i$ contains $\bs{a}_i^\top$ as its rows. Since $s$ operates elementwise, we can again rewrite the \gn\ as
    \begin{equation}
        \bs{A}^\top \bs{U} s(\bs{V} (\bs{A} \x + \bs{b}) + \bs{\beta})
        \label{eq:one_layer_gradnet_with_nonpolynomial_activation}
    \end{equation}
    where the $(N\cdot\prod_{i=1}^N m_i) \times d$ matrix $\bs{A}$ vertically stacks the $\bs{A}_i$ matrices and the vector $\bs{\beta}$ stacks the $\bs{\beta}_i$ vectors. The diagonal matrices $\bs{U}, \bs{V}$ respectively have $\mathrm{diag}(\bs{u}_i), \mathrm{diag}(\bs{v}_i)$ along their diagonals. Therefore, \eqref{eq:one_layer_gradnet_with_nonpolynomial_activation} is a \gn\ in~\eqref{eq:single_layer_gn} with intermediate diagonal matrices corresponding to scaled elementwise nonpolynomial activations.
\end{proof}
The proof of Thm.~\ref{thm:gn_ridges_universal} uses the fact that neural networks with elementwise nonpolynomial activations universally approximate continuous functions on compact subsets of $\R$~\cite{Nonpolynomial_LESHNO1993861}. In fact, the proof permits any elementwise activation that is a universal approximator.
{\begin{corollary}
    Let $\mathcal{F}$ be functions in $C^1([0,1]^d)$ expressible as a finite sum of ridge functions. Let $\sigma$ be an elementwise activation where each $\sigma_i \in C^0(\R)$ is a universal approximator of continuous functions on compact subsets of $\R$. \gns\ in~\eqref{eq:single_layer_gn} with activation $\sigma$ universally approximate $\nabla \mathcal{F}$.
    \label{cor:gn_sum_of_ridges}
\end{corollary}
Cor.~\ref{cor:gn_sum_of_ridges} also applies to other architectures for approximating gradients. For example, variational networks proposed in \cite{hammernik2018learning} use Gaussian radial basis functions (RBFs) for elementwise activations, making them universal approximators of sums of ridge functions \cite{RBF}. Similarly, the weakly convex ridge regularizers in \cite{goujon2024learning}, use elementwise learnable spline activations, which are also dense in the space of continuous functions.}

Now, we shift focus to convex ridges and similarly prove that \mgns\ with elementwise, nondecreasing activations can learn the gradient of any function expressible as the sum of \textit{convex} ridges. The proof proceeds as follows: we first show that \mgns\ in Prop.~\ref{prop:single_layer_mgn} with elementwise activations can universally approximate monotone functions on $\R$; similar to the proof of Thm.~\ref{thm:gn_ridges_universal}, we then employ this \mgn\ to learn $\psi_i'$, the monotone derivatives of the convex ridges. These results also motivate compositions of \mgns\ in Sec.~\ref{sec:transformations} and deeper networks in Sec.~\ref{sec:cmgn}.
{\begin{lemma}[Universal Approximation for Scalar Monotone Functions]
    Let $\mathcal{F}$ be nondecreasing functions in $C^0([0,1])$. Let $\sigma$ be an elementwise activation where each $\sigma_k \in C^0(\R)$ is bounded, nondecreasing, and has finite asymptotic end behavior.
    \mgns\ of the form \eqref{eq:single_layer_gn} with activation $\sigma$ universally approximate $\mathcal{F}$.
    \label{lemma:universalR_to_R}
\end{lemma}}
\begin{proof}
Let $f \in \mathcal{F}$ and, without loss of generality, let each $\sigma_k(x) : \R \to [0,1]$ with $\lim_{x \to -\infty} = 0$ and $\lim_{x \to \infty} = 1$. We uniformly partition $[0,1]$ into subintervals $\mathcal{I}_{n,k} =[\frac{k-1}{2^{n}}, \frac{k}{2^{n}}]$ for $k \in \{1,2,\dots,2^n\}$ and consider the approximators
\begin{align}
    g_{n,t}(x)&= f(0)  + \sum_{i = 1}^{2^{n}}\Delta_{n,i} \sigma_i(t (2^{n+1} x -  2i  +  1))
\end{align}
where $t\geq 1$, each $\sigma_i$ satisfies conditions given in the lemma, and $\Delta_{n,i} = f\left(\frac{i}{2^{n}}\right)-f\left(\frac{i-1}{2^{n}}\right) \geq 0$.
Next, we bound
\begin{align}
    \max_{x \in [0,1]} | f(x) - g_{n,t}(x) | =
    \max_k \max_{x \in \mathcal{I}_{n,k}} | f(x) - g_{n,t}(x) | \label{eq:max_over_subintervals}
\end{align}
To do so, we introduce $\widetilde{g}_{n,t}$, which has the same form as $g_{n,t}$ except each $\sigma_i$ is replaced by $\widetilde{\sigma}$:
\begin{align}
    \widetilde{\sigma}(x) = \begin{cases}
        0 \text{ if } x\leq-1; \;\;
        x \text{ if } -1< x < 1; \;\;
        1 \text{ if } x\geq 1
    \end{cases}
\end{align}
In each of the $2^n$ subintervals, $\widetilde{g}_{n,t}$ interpolates $f$.
For a specific $k$ on the RHS of~\eqref{eq:max_over_subintervals}, triangle inequality implies
\begin{align}
    \max_{x \in \mathcal{I}_{n,k}} | f(x) - g_{n,t}(x) | &\leq \max_{x \in \mathcal{I}_{n,k}} | f(x) - \Tilde{g}_{n,t}(x) | + | \Tilde{g}_{n,t}(x) - g_{n,t}(x) | \label{eq:triangle_mgn_R_to_R}
\end{align}
Let $\epsilon>0$. For any $n\in\mathbb{N}, t\geq 1$, monotonicity of $f$ implies that the first term on the RHS of~\eqref{eq:triangle_mgn_R_to_R} is bounded as follows:
\begin{align}
    \max_{x \in \mathcal{I}_{n,k}} | f(x) - \widetilde{g}_{n,t}(x) | \leq \Delta_{n,k}
\end{align}
Now, by the Heine-Cantor Theorem~\cite{Heine_Cantor}, $f$ is uniformly continuous on $[0,1]$.
Hence, let $n\geq n_0\in\mathbb{N}$ such that $\forall k,~\Delta_{{n_0},k} = f\left({k}/{2^{n_0}}\right)-f\left({(k-1)}/{2^{n_0}}\right) < \epsilon /3$.
For each $k$, the second term on the RHS of~\eqref{eq:triangle_mgn_R_to_R} can be bounded by separately considering the absolute difference between $\Tilde{g}(x)$ and $g(x)$ in three different intervals: $i < k$, $i=k$, and $i > k$.
\begin{align}
    \max_{x \in \mathcal{I}_{n,k}} | \Tilde{g}_{n,t}(x) - g_{n,t}(x)| &\leq \max_{x \in \mathcal{I}_{n,k}}
    \biggr\{\sum_{i < k} \Delta_{n,i}[1-\sigma_i(t (2^{n+1} x - 2i + 1))]\\
    &\hspace{20mm} +\Delta_{n,k}|\widetilde{\sigma}(t (2^{n+1} x - 2k + 1))-\sigma_k(t (2^{n+1} x - 2k + 1))|&&\\
    &\hspace{20mm}+ \sum_{i > k} \Delta_{n,i}\sigma_i(t (2^{n+1} x - 2i + 1))\biggr\}&&
    \label{eq:bound_on_g_minus_gtilde}
\end{align}
By boundedness of $\widetilde{\sigma}$ and $\sigma_k$, we have $\forall x,$ $|\widetilde{\sigma}(x)-\sigma_k(x)| \leq 1$ and the middle term in~\eqref{eq:bound_on_g_minus_gtilde} is upper-bounded by
$\Delta_{n,k}<\epsilon / 3$.
Using the aforementioned value of $n$, there exists corresponding $t>1$ such that the sum of the remaining terms of~\eqref{eq:bound_on_g_minus_gtilde} is bounded by $\epsilon/3$. Therefore, $\max_{x \in [0,1]} | f(x) - g_{n,t}(x) |<\epsilon$ and $g_{n,t}$ universally approximates $\mathcal{F}$. We note that $g_{n,t}$ can be written in the form of~\eqref{eq:single_layer_gn} with nondecreasing activations: 
\begin{equation}
    g_{n,t} = f(0)  + \bs{1}^\top \overline{\sigma}(\bs{1} x + \bs{b}) \label{eq:mgn_r_to_r_form}
\end{equation}
where $\bs{b}_k = (-2k  +  1)/2^{n+1}$ and $\overline{\sigma}$ is an elementwise activation composed of $\overline{\sigma}_k(x) = \Delta_{n,k}\sigma_k(2^{n+1}t x)$.
\end{proof}
Lemma~\ref{lemma:universalR_to_R} states that a sufficiently wide \mgn\ with scaled activations, like sigmoids, can closely approximate any continuous, monotone function on $\R$. {It shows that a broader set of functions than the linear splines considered in \cite{EPFL} is dense in the space of continuous nondecreasing functions on $\R$}. Next, we give universal approximation results for sums of \textit{convex} ridges {that generalize Prop. III.5 of \cite{EPFL}}. These results use Lemma~\ref{lemma:universalR_to_R} and are similar to the proofs in Thm.~\ref{thm:gn_ridges_universal} and Cor.~\ref{cor:gn_sum_of_ridges} that pertain to sums of general ridge functions.

{
\begin{theorem}[Universal Gradient Approximation for Sums of Convex Ridges]
Let $\mathcal{F}$ be convex functions in $C_\times^1([0,1]^d)$ expressible as a finite sum of convex ridge functions. Let $\sigma$ be an elementwise activation where each $\sigma_k \in C^0(\R)$ is bounded, nondecreasing, and has finite asymptotic end behavior.
\mgns\ in Prop.~\ref{prop:single_layer_mgn} with scaled activation $\sigma$ universally approximate $\nabla \mathcal{F}$.
    \label{thm:mgn_ridges_universal}
\end{theorem}}
\begin{proof}
$F$ is of the form~\eqref{eq:ridges}, where each $\psi_i$ is convex. By Lemma~\ref{lemma:universalR_to_R}, an \mgn\ of the form~\eqref{eq:mgn_r_to_r_form} can approximate each nondecreasing $\psi_i'$ on a compact subset of $\R$ within arbitrary error. The remainder of the proof follows from the proof of Thm.~\ref{thm:gn_ridges_universal}.
\end{proof}
Similar to Cor.~\ref{cor:gn_sum_of_ridges}, the proof of Thm.~\ref{thm:mgn_ridges_universal} permits the activation to be any elementwise function that is a universal approximator of nondecreasing functions.

{\begin{corollary}
    Let $\mathcal{F}$ be convex functions in $C_{\times}^1([0,1]^d)$ expressible as the finite sum of convex ridge functions. Let $\sigma$ be an elementwise activation where each $\sigma_i \in C^0(\R)$ is a universal approximator of continuous, nondecreasing functions on compact subsets of $\R$. \mgns\ in~\eqref{eq:single_layer_gn} with activation $\sigma$ universally approximate $\nabla \mathcal{F}$.
    \label{cor:mgn_sum_of_ridges}
\end{corollary}}

\subsection{Gradients of (convex monotone) transformations of sums of (convex) ridges}
\label{sec:transformations}
This section presents architectures that parameterize a broader subset of (monotone) gradient functions than those discussed in Sec.~\ref{sec:ridges}. We specifically examine methods for composing vector-valued gradient networks with scalar-valued networks to learn gradients of transformations of sums of ridges. We then adapt these networks to learn gradients of convex, monotone transformations of sums of convex ridges. 

\begin{definition}[(Convex monotone) transformation of a sum of (convex) ridge functions]
    $F \in C^k(\R^d)$ is a transformation of a finite sum of ridge functions if it can be written as 
    \begin{equation}
        F(\x) = T\left(\sum_{i=1}^N \psi_i(\bs{a}_i ^\top \x + b_i)\right)
        \label{eq:transformed_ridges}
    \end{equation}
    where $\psi_i, T \in C^k(\R)$. $F$ is a convex, monotone transformation of a finite sum of convex ridge functions if $\psi_i, T \in C_{\times}^k (\R)$ and $T$ is additionally nondecreasing\cite{boyd2004convex}.
    \label{def:transformed_ridge_function}
\end{definition}
Def.~\ref{def:transformed_ridge_function} closely resembles Def.~\ref{def:ridge_function} with the addition of a transformation $T$ applied to the sum of ridges. Before describing methods for learning gradients of such functions, we prove a useful lemma which shows that approximating $\nabla F$ with $\nabla G$ implies that $G$ can approximate $F$.

\begin{lemma}
    Let $\epsilon > 0$, and $f,\hat{f}: [0,1]^d \to \R$ be differentiable.
    If 
    $\sup_{\x\in[0,1]^d}\|\nabla \hat{f}(\x) - \nabla f(\x)\| \leq \epsilon$
    and there exists $\x_0\in[0,1]^d$ with known $f(\x_0)$, then $\hat{f}$ can approximate $f$: 
    $\sup_{\x\in[0,1]^d}|\hat{f}(\x) - f(\x)| \leq \epsilon \sqrt{d}$
    \label{lemma:mean_value_theorem}
\end{lemma}
\begin{proof}
    Let $p(\x) = \hat{f}(\x) - f(\x)$, where $\hat{f}(\x)$ is the antiderivative of the known $\nabla \hat{f}$ with a constant of integration determined by $\x_0$. Note that $\nabla p(\x) = \nabla \hat{f}(\x) - \nabla f(\x)$. By the multivariate mean value theorem~\cite{rudin1964principles}, $\forall \x,\y\in[0,1]^d,~\exists c\in(0,1)$ such that $\z = (1-c)\x + c\y$ and 
    \begin{align}
        p(\x) - p(\y) &= \nabla p(\z)^\top (\x - \y)\\
        \hat{f}(\x) - f(\x) - \hat{f}(\y) + f(\y) &= (\nabla \hat{f}(\z) - \nabla f(\z))^\top (\x - \y)
    \end{align}
    Let $\y = \x_0$ so that $-\hat{f}(\y) + f(\y) = 0$. By the Cauchy-Schwarz inequality and an upper bound on $\|\x - \x_0\|$,
    \begin{align}
        |\hat{f}(\x) - f(\x)| &\leq \|\nabla \hat{f}(\z) - \nabla f(\z)\|\|\x - \x_0\| \leq \epsilon\sqrt{d}
    \end{align}
    We note that if $\nabla \hat{f}_n$ is a sequence of functions that converges to $\nabla f$, with a corresponding sequence of $\epsilon_n > 0$ converging to $0$, then the sequence of antiderivatives $\hat{f}_n$ converges to $f$ regardless of the dimension $d$.
\end{proof}

Using the above lemma, we describe a method for learning gradients of functions specified in Def.~\ref{def:transformed_ridge_function}.

\begin{theorem}[Universal Gradient Approximation for Transformed Sums of Ridges]
    Let $\mathcal{F}$ be functions in $C^1([0,1]^d)$ expressible as  transformations of finite sums of ridge functions. Let $G \in C^1([0,1]^d)$, $g(\x) = \nabla_{\x} G(\x)$, and $\gamma \in C^0(\R)$.
    Let $h$ be of the form
    \begin{align}
        h(\x) &= \gamma(G(\x) + \beta)\cdot g(\x)
        \label{eq:composition_for_transformations}
    \end{align}
    where $\beta\in\R$ is a learnable bias. If $\gamma$ universally approximates continuous functions on any compact subset of $\R$ and $g$ universally approximates gradients of sums of ridge functions on $[0,1]^d$,
    then functions of the form $h$ are \gns\ that universally approximate $\nabla\mathcal{F}$.
    \label{thm:universal_transf_ridges}
\end{theorem}
\begin{proof}

If $F \in \mathcal{F}$, then $\nabla F$ is of the same form as $h$:
\begin{align}
    \nabla F &= T'\left(\sum_{i=1}^N \psi_i(\bs{a}_i^\top \x + b_i)\right)\cdot\left(\sum_{i=1}^N \bs{a}_i\psi'_i(\bs{a}_i^\top \x + b_i)\right)
    \label{eq:transf_grad}
\end{align}
By definition, $\gamma$ can approximate $T'$ and $g(\x)$ can approximate $\sum_{i=1}^N \bs{a}_i\psi_i'(\bs{a}_i^\top \x + b_i)$ within arbitrary error. By Lemma~\ref{lemma:mean_value_theorem}, we get that $G(\x) + \beta$ can simultaneously approximate $\sum_{i=1}^N \psi_i(\bs{a}_i^\top \x + b_i)$.
Thus, each function in~\eqref{eq:composition_for_transformations} can learn the corresponding function in~\eqref{eq:transf_grad} within arbitrary error. Furthermore,~\eqref{eq:transf_grad} is a \gn\ as it corresponds to $\nabla\Gamma(G(\x) + \beta)$, where, by continuity of $\gamma$ on $\R$, $\Gamma$ is the antiderivative of $\gamma$ on the compact codomain of $G(\x) + \beta$.
\end{proof}
Thm.~\ref{thm:universal_transf_ridges} demonstrates that gradients of functions specified by Def.~\ref{def:transformed_ridge_function} can be learned by parameterizing $\gamma$ as a scalar-valued neural network (since neural networks are universal function approximators~\cite{cybenko1989approximation}) and taking $g(\x)$ to be a network of the form given in Thm.~\ref{thm:gn_ridges_universal}, which can approximate gradients of sums of ridges to arbitrary error. Furthermore, Cor.~\ref{cor:gn_sum_of_ridges} allows us to parameterize $g$ with a single hidden layer neural network whose antiderivative is simple to compute, hence we can easily obtain the form for $G$, the antiderivative of $g$, in Thm.~\ref{thm:universal_transf_ridges}.


We now extend Thm.~\ref{thm:universal_transf_ridges} to specifically use \mgns\ to learn gradients of \textit{monotone convex} transformations of sums of \textit{convex} ridges. The parameterization requires a function $\gamma$ that universally approximates monotone nonnegative functions. The following construction of $\gamma$ leverages the universality of the \mgn\ on $\R$, as proved in Lemma~\ref{lemma:universalR_to_R}. 

\begin{lemma}
    Let $\mathcal{F}$ be the set of nondecreasing functions in $C^0([0,1])$. Let $\tau$ universally approximate nondecreasing functions in $C^0([0,1])$ and $\rho$ be uniformly continuous and strictly increasing. Then functions of the form $\rho (\tau(x))$ universally approximate $\mathcal{F}$.
    \label{lemma:universal_nonnegative}
\end{lemma}
\begin{proof}
Let $f \in \mathcal{F}$ and $\epsilon > 0$. Since $\rho$ is strictly increasing, it is invertible. Since $\rho^{-1}(f(x))$ is continuous and nondecreasing, there exists $\tau$ that approximates $\rho^{-1}(f(x))$ with arbitrary error $\delta>0$. By uniform continuity of $\rho$ there exists $\delta > 0$ such that $ \forall x,~|\tau(x) - \rho^{-1}(f(x))| < \delta$ implies $|\rho(\tau(x)) - f(x)|<\epsilon$. 
\end{proof}

Lemma~\ref{lemma:universal_nonnegative} demonstrates how composing a universal approximator of differentiable monotone functions with another function can yield a universal approximator of differentiable monotone functions that map to a subset of $\R$. Lipschitz functions are uniformly continuous, thus taking $\rho$ to be softplus and $\tau$ to be an \mgn\ in Lemma~\ref{lemma:universalR_to_R} yields a universal approximator of differentiable, nonnegative, monotone functions $\gamma$. We extend Thm.~\ref{thm:universal_transf_ridges} with this construction for $\gamma$.

\begin{theorem} [Universal Gradient Approximation for Transformed Sums of Convex Ridges]
Let $\mathcal{F}$ be functions in $C_{\times}^1([0,1]^d)$ expressible as convex monotone transformations of finite sums of convex ridge functions. Let $h$ be as in~\eqref{eq:composition_for_transformations}, where $G\in C^1_\times([0,1]^d)$, $\beta\in\R$ is a learnable bias, and $\gamma \in C^0(\R)$ is nondecreasing and nonnegative. 
Let $\gamma$ universally approximate nondecreasing and nonnegative functions in $C^0(\R)$ on compact subset of $\R$. Let $g(\x)=\nabla_{\x} G(\x)$ universally approximate gradients of sums of \textit{convex} ridge functions on $[0,1]^d$. Functions of the form $h$ are \mgns\ that universally approximate $\nabla\mathcal{F}$.
\label{thm:universal_mgn_transf_ridges}
\end{theorem}
\begin{proof}
Let $F\in\mathcal{F}$. Then $\nabla F(\x)$ is as in~\eqref{eq:transf_grad}, where convexity of $\psi_i$ and convex monotonicity of $T$ respectively imply that the $\psi_i'$ are monotone and $T'$ is monotone nonnegative. 
Lemma~\ref{lemma:universal_nonnegative} provides a construction for continuous $\gamma$ that can approximate $T'$ within arbitrary error. For instance, $\tau$ can be as described in Lemma~\ref{lemma:universalR_to_R} and $\rho$ can be any uniformly continuous, nondecreasing function such as softplus or ReLU. The remainder of the proof is nearly identical to that of Thm.~\ref{thm:universal_transf_ridges} with $G, g$ appropriately corresponding to convex ridge functions. 
\end{proof}

Thm.~\ref{thm:universal_mgn_transf_ridges} demonstrates that $h$ can be parameterized using an \mgn\ to yield the gradient of any function expressible as the convex monotone transformation of a sum of convex ridges. As in Thm.~\ref{thm:universal_transf_ridges}, knowledge of the antiderivative of the \mgn\ $g(\x)$ is required. However, as shown in Cor.~\ref{cor:mgn_sum_of_ridges}, there exist \mgns\ with known antiderivatives that can learn the gradient of any sum of convex ridges.

{\subsection{Learning Subgradients of Convex Functions}
\label{sec:subgradients}
The previous subsections discussed parameterizing and learning gradients of {differentiable} convex functions. In this section, we relax the differentiability assumption and show that \mgns\ can effectively characterize subgradients of \textit{subdifferentiable} convex functions.
\begin{definition}[Subgradient, Subdifferential]
    A subgradient of $F \in C^{0}_{\times}(\R^d)$ at a point $\x$ is any $\v$ that satisfies 
    \begin{equation}
         \v^\top (\y - \x) \leq F(\y) - F(\x),\;\forall \y \in \R^d
    \end{equation}
    The subdifferential of $F$ at $\x$, denoted by $\partial F(\x)$, is the set of all subgradients of $F$ at $\x$.
\end{definition}
If $F$ is convex and differentiable at a point $\x$, then the subgradient at $\x$ is uniquely the gradient of $F$ at $\x$. We also note that any convex function is subdifferentiable on the interior of its domain~\cite{rockafellar1970convex}. To learn subgradients of convex functions, we use the following lemma.  
\begin{lemma}(\!\cite{rockafellar1970convex} Theorem 24.5)
Let $\mathcal{S} \subseteq \R^d$ be open, convex and let $F\in C^0_\times(\mathcal{S})$ be finite. Let $G_i \in C^0_\times(\mathcal{S})$ be a sequence of finite functions converging pointwise to $F$ on $\mathcal{S}$. For any $\epsilon > 0$, there exists index $i_0$ such that $\forall i \geq i_0,\;\partial G_{i}(\x) \subset \partial F(\x) + \epsilon \mathcal{B}$ where $\mathcal{B}$ is the Euclidean unit ball of $\R^d$.
\label{lemma:rockafeller_subgradient}
\end{lemma}
The lemma states that if a sequence of convex approximators converges pointwise to a target convex function, then the subdifferentials of the approximators become increasingly similar to those of the target function. If the approximators are additionally differentiable, the subdifferential $\partial G_i(\x)$ is the gradient $\nabla G_i(\x)$ and can be made arbitrarily close to an element in $\partial F(\x)$. Therefore, the existing results in Sec.~\ref{sec:universal_approximation} based on Lemma~\ref{lemma:rockafeller} can be readily adapted to handle subdifferentiable convex functions. 




}
\section{{\gn\ Architecture Variants}}
\label{sec:enhancements}
In this section, we propose specific neural network architectures for parameterizing \gns\ and \mgns. The architectures we propose exhibit universal approximation properties while empirically being more amenable for optimization (as seen in Sec.~\ref{sec:experiments}). In the following subsections, we initially describe conditions under which our proposed architectures are \gns\ and then discuss stricter conditions under which the networks become \mgns.

\subsection{Modular (Monotone) Gradient Networks (\gnm, \mmgn)}
\label{sec:mmgn}

\begin{figure}[htb]
    \centering
    \includegraphics[width=0.6 \linewidth,  trim={2em 27em 37em 5em}, clip]{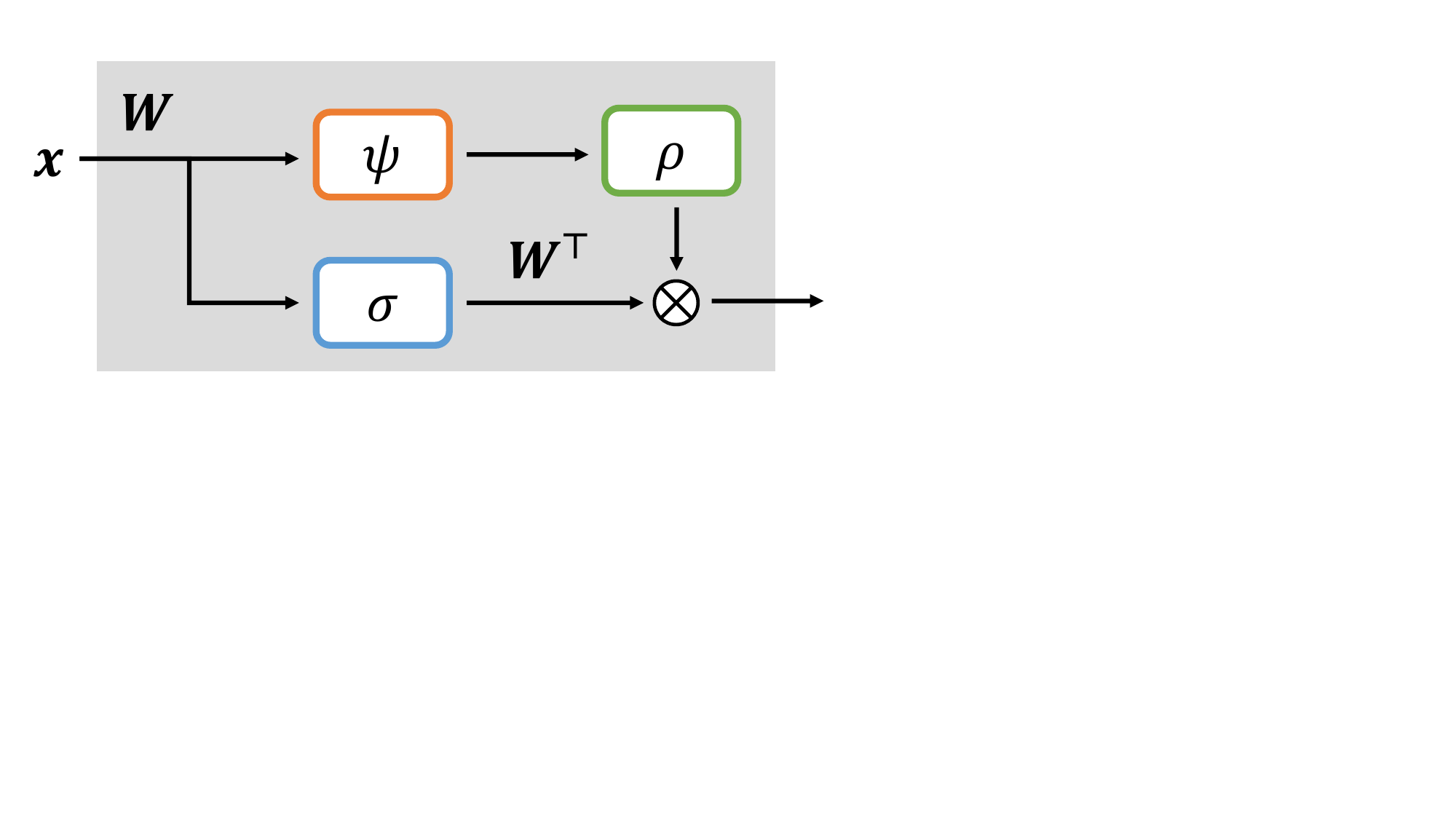}
    \caption{Single module of the modular gradient network (\gnm) defined in~\eqref{eqn:m-mgn_a}-\eqref{eqn:m-mgn_b}.}
    \label{fig:mmgn_module_arch}
\end{figure}

Here we describe modular gradient networks (\gnm) and their monotone counterparts (\mmgn), which are motivated by the discussion of Lemma~\ref{lemma:lse_approx} in Sec.\ref{sec:universal_approximation}. These networks achieve wide, modular architectures by respectively using \gns\ and \mgns, with different weight matrices, as building blocks. 
\gnms\ (\mmgns) can universally approximate a broader range of functions than just those expressible as the sum of (convex) ridges, and thus are more expressive than existing methods \cite{FieldsOfExperts,donoho1994ideal,patchbasedsparsemodelstohigherorderMRFs,hammernik2018learning, EPFL,goujon2024learning}. \gnms\ universally approximate gradients of $L$-smooth functions and transformations of sums of ridges, as shown in Thm.~\ref{thm:gn_universal_approx} and Thm.~\ref{thm:universal_transf_ridges} respectively. Similarly, \mgns\ universally approximate gradients of convex functions and gradients of transformations of sums of ridges, as demonstrated in Thm.~\ref{thm:universal_approximation} and Thm.~\ref{thm:universal_mgn_transf_ridges}. 

The design of \gnms\ leverages the facts that a linear combination of \gns\ and a composition of a \gn\ with a scalar-valued, differentiable function both yield \gns, as described respectively in Rem.~\ref{rem:gn_linear_combinations} and Thm.~\ref{thm:universal_transf_ridges}. 
The following equations define the \gnm\ and generalize the architecture we previously proposed in \cite{Our_ICASSP_Paper}:
\begin{align}
\label{eqn:m-mgn_a}
\z_m &= \W_m\x + \b_m \\
\gnm(\x)&=\bs{a}+\sum_{m=1}^M\rho_m(\phi_m\left(\z_m \right))\W_m^\top\sigma_m\left(\z_m\right)
    \label{eqn:m-mgn_b}
\end{align}
where $\bs{a}$ denotes a learnable bias and $M$ denotes the number of modules -- a hyperparameter tunable based on the application. $\sigma_m, \W_m$ and $\b_m$ respectively denote the activation function, weight matrix, and bias corresponding to the $m^{th}$ module. The $\phi_m$ are vector-to-scalar functions and $\rho_m$ are scalar-to-scalar functions. The block diagram corresponding to a single module output is shown in Fig.~\ref{fig:mmgn_module_arch}. In the following theorem, we provide conditions for $\phi_m, \rho_m,$ and $\sigma_m$ under which the \gnm\ is a \gn.

\begin{theorem}[\gnm\ conditions]
    If $\rho_m, \phi_m, \sigma_m$ are all differentiable and $\sigma_m = \nabla \phi_m$, the \gnm\ in~\eqref{eqn:m-mgn_a}-\eqref{eqn:m-mgn_b} is a \gn.
\end{theorem}
\begin{proof}
    Consider the $m^{th}$ module $\rho_m(\phi_m(\z_m))\W_m^\top \sigma_m(\z_m)$. The Jacobian of the $m^{th}$ module with respect to $\x$ is 
    \begin{align}
        \label{eq:gnm_jac}
        \J_m(\x)&= \rho_m'(\phi_m(\z_m))(\W_m^\top \sigma_{m}(\z_m))(\W_m^\top \sigma_{m}(\z_m))^\top + \rho_m(\phi_m(\z_m))\W^\top_m\J_{\sigma_m}(\z_m)\W_m 
    \end{align}
    The first term on the RHS is a Gram matrix scaled by $\rho'_m(\phi_m(\z_m))$, and hence symmetric. The second term on the RHS is symmetric by Lemma~\ref{lemma:gn_jac_sym} since $\J^\top_{\sigma_m} = \bs{H}_{\phi_m}$. Thus, each module is a \gn\ and the result holds by Rem.~\ref{rem:gn_linear_combinations}.
\end{proof}

Next, we introduce the monotone \gnm:
\begin{corollary}[\mmgn\ conditions]
    The \gnm\ defined by~\eqref{eqn:m-mgn_a}-\eqref{eqn:m-mgn_b} is an \mgn, and referred to as \mmgn\ if, for each module,
    \begin{enumerate}[a)]
        \item $\phi_m : \R^d \to I_m \subseteq \R$ is convex, twice differentiable
        \item $\rho_m: I_m \to\R_{\geq 0}$ is differentiable and monotone
        \item $\sigma_m = \nabla \phi_m$
    \end{enumerate}
    \label{cor:mmgn_conditions}
\end{corollary}
\begin{proof}
    The Jacobian of the $m^{th}$ module is given by~\eqref{eq:gnm_jac}. Convexity of $\phi_m$ implies $J_{\sigma_m}$ is everywhere PSD. Since $\rho_m$ is nonnegative-valued, the second term on the RHS of~\eqref{eq:gnm_jac} is PSD. Monotonicity of $\rho_m$ implies that $\rho_m'$ is everywhere nonnegative, making the first term of~\eqref{eq:gnm_jac} also PSD. Hence, each module is an \mgn\ and the result holds by Rem.~\ref{rem:mgn_conical_combinations}.
\end{proof}

There are several suitable choices for $\rho_m, \phi_m,$ and $\sigma_m$ in Cor.~\ref{cor:mmgn_conditions}. For example: $\phi_m(\x) = \lse(\x)$, $\sigma_m(\x) = \mathrm{softmax}(\x)$ and $\rho_m$ any differentiable, monotone nonnegative function on $\R$ such as softplus. As another example, one can take $\sigma_m$ to be the elementwise $\mathrm{sigmoid}$ activation function with $\phi_m(\x) = \sum_{i} \mathrm{softplus}(x_i)$ and $\rho(x) = \alpha x + \beta, \text{ where } \alpha > 0$. 

\subsection{Cascaded (Monotone) Gradient Networks (\gnc, \cmgn)}
\label{sec:cmgn}
\begin{figure}[htpb]
    \centering
    \includegraphics[width=0.8\linewidth,  trim={1em 32em 15em 1em}, clip]{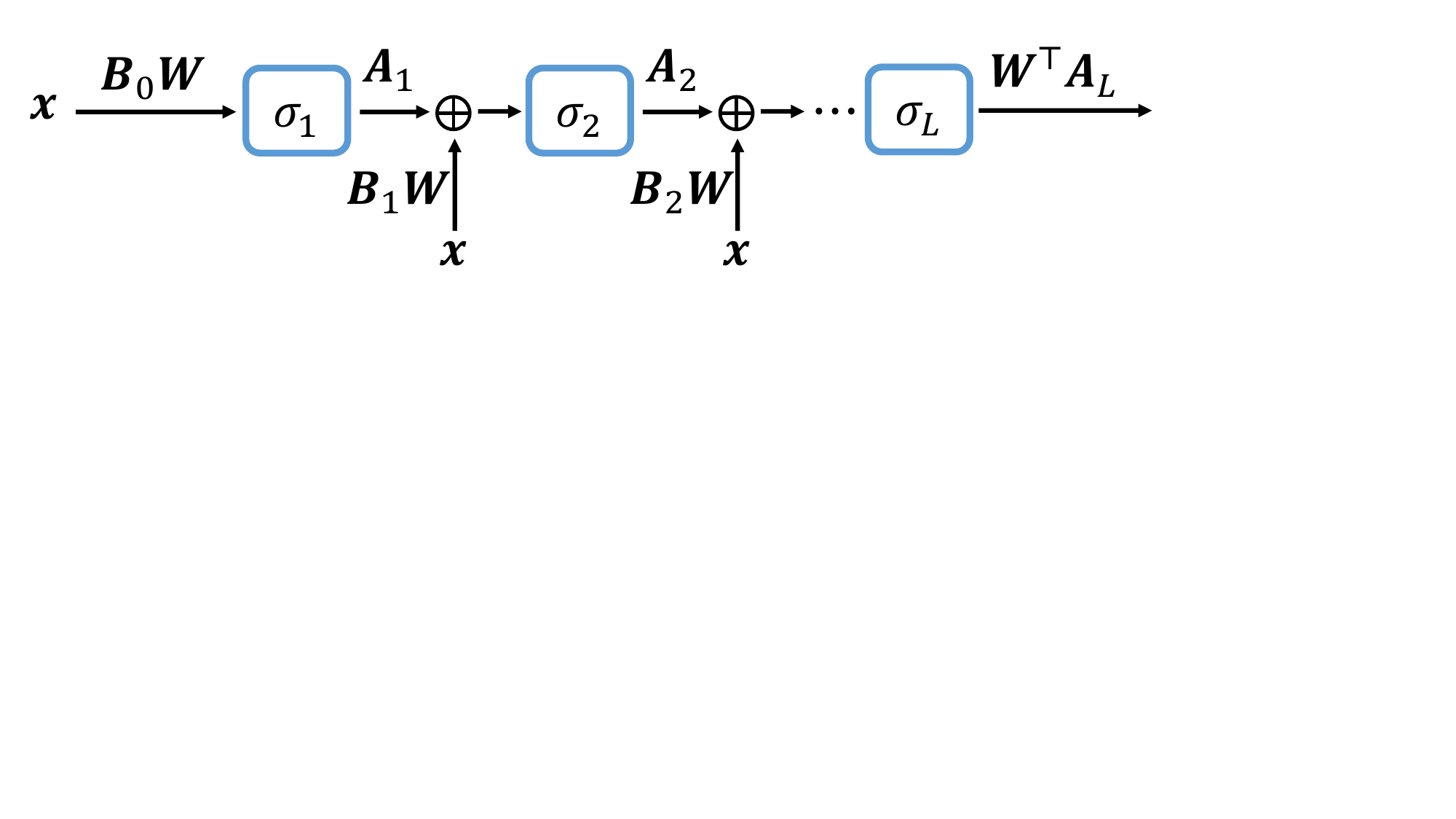}
    \caption{Cascaded gradient network (\gnc), defined in~\eqref{eqn:cmgn_start}-\eqref{eqn:cmgn_end}, with $\bs{A}_\ell = \diag{\balpha_\ell},\bs{B}_\ell = \diag{\bbeta_\ell}$.}
    \label{fig:cmgn_arch}
\end{figure}
In this section we generalize the \gn\ and \mgn, described by~\eqref{eq:single_layer_gn}, to achieve \textit{deeper} networks. The architecture is inspired by Thm.~\ref{thm:gn_ridges_universal}, which states that the \gn\ in~\eqref{eq:single_layer_gn} with elementwise neural network activations can universally approximate the gradient of any function expressible as the sum of ridges. {The networks we propose in this section, namely \gnc\ and \cmgn, are consequently as expressive as the methods in \cite{FieldsOfExperts,donoho1994ideal,patchbasedsparsemodelstohigherorderMRFs,hammernik2018learning, EPFL,goujon2024learning} while featuring a novel and more efficient parameterization of the elementwise activation functions.} We propose a cascaded gradient network (\gnc), illustrated in Fig.~\ref{fig:cmgn_arch} and defined by the layerwise equations:
\begin{align}
    \z_0 &= \bbeta_0 \odot \W\x+\b_0 \label{eqn:cmgn_start}\\
    \z_\ell &= \bbeta_\ell \odot \W\x + \balpha_{\ell} \odot \sigma_\ell(\z_{\ell-1}) + \b_\ell\\
    \gnc(\x) &= \W^\top\left[\balpha_{L} \odot \sigma_L\left(\z_{L-1}\right)\right] + \b_{L}
    \label{eqn:cmgn_end}
\end{align}
where $\z_{\ell},\b_\ell$, and $\sigma_\ell$ respectively denote the output, bias, and activation function at layer $\ell$ of the network and $\odot$ is the Hadamard (entrywise) vector product. The weight matrix $\W$ is shared across all layers whereas the intermediate scaling weight vectors $\balpha_\ell, \bbeta_\ell$ are unique to each layer.
\begin{theorem}[\gnc\ conditions]
    The \gnc\ in~\eqref{eqn:cmgn_start}-\eqref{eqn:cmgn_end} is a \gn\ if the $\sigma_\ell$ are differentiable elementwise activations.
    \label{thm:gnc_conditions}
\end{theorem}
\begin{proof}
Let $\bs{A}_\ell = \diag{\balpha_\ell},\bs{B}_\ell = \diag{\bbeta_\ell}$. The Jacobian of $\gnc(\x)$ with respect to $\x$ is:
\begin{align}
 \J_{\gnc}(\x) &= \W^\top \bs{D}\W
\label{eq:cmgn_jac_a}\\
 \bs{D} &= \sum_{\ell=1}^{L}
 \left(\prod_{i=0}^{L-\ell}\bs{A}_{L-i}\J_{\sigma_{L-i}}\left(\z_{L-i-1}\right)\right)\bs{B}_{\ell-1}
 \label{eq:cmgn_jac_b}
\end{align}
The activation Jacobians $\J_{\sigma_\ell}$  are diagonal matrices since each $\sigma_\ell$ operates elementwise on the input. Therefore $\bs{D}$ is a diagonal matrix and $\J_{\gnc}$ is symmetric.
\end{proof}
The activation $\sigma_\ell$ in Thm.~\ref{thm:gnc_conditions} can be a fixed function applied elementwise, e.g., tanh, softplus, and sigmoid, or a \textit{learnable} function that operates on individual elements of the input vector. Thus, the proposed approach also permits elementwise activations parameterized by scalar-valued neural networks previously highlighted in Cor.~\ref{cor:gn_sum_of_ridges}. The proposed cascaded networks also remain \gns\ if all $\bbeta_\ell$ are 0. However, as shown in \cite{he2016deep, li2018visualizing}, skip connections accelerate training as they address the vanishing gradient problem in deep networks and smoothen the loss landscape. Next, we introduce the monotone variant of the \gnc.
\begin{corollary}[\cmgn\ conditions]
    The \gnc\ in~\eqref{eqn:cmgn_start}-\eqref{eqn:cmgn_end} is an \mgn, and referred to as \cmgn, if, at each layer, the scaling weights $\balpha_\ell, \bbeta_\ell$ are nonnegative and the $\sigma_\ell$ are differentiable, monotonically-increasing elementwise activation functions.
    \label{cor:cmgn_conditions}
\end{corollary}
\begin{proof}
Consider $\J_{\gnc}$ from~\eqref{eq:cmgn_jac_a}-\eqref{eq:cmgn_jac_b}. By the assumptions in the theorem, for all $\ell$, $\bs{A}_\ell, \bs{B}_\ell$ are nonnegative diagonal matrices and the activation Jacobians $\J_{\sigma_\ell}$ are also nonnegative diagonal matrices. Thus, $\bs{D}$ is everywhere PSD and it follows that $\bs{J}_{\gnc}$ is everywhere PSD.
\end{proof}

The condition in Cor.~\ref{cor:cmgn_conditions} states that the activations $\sigma_\ell$ of an \cmgn\ should be monotonically increasing. Most popular elementwise activation functions, e.g., tanh, softplus, and sigmoid, satisfy this requirement. Cor.~\ref{cor:cmgn_conditions} also permits the use of learned, elementwise \mgn\ activations as discussed in Thm.~\ref{thm:mgn_ridges_universal}. When considering \gnc\ and \cmgn\ on compact domains in practice, the activations $\sigma$ need only be continuous rather than differentiable. The requirement that the intermediate elementwise scaling weights of \cmgn\ at each layer are nonnegative vectors is easy to parameterize in practice. Moreover, we empirically observe in Sec.~\ref{sec:experiments} that imposing nonnegativity constraints on the intermediate weight vectors does not impair optimization or final performance. 

\section{Experiments}
\label{sec:experiments}

{\subsection{Gradient Field}
\label{sec:gradient_field_experiments}
{We demonstrate the proficiency of our networks in learning gradients of scalar functions over the unit cube $[0, 1]^d$. To visualize the error incurred by each network, we consider a low-dimensional setting with $d=2$ in Sec.~\ref{sec:2d_gradfield}. We find that our methods result in error plots with fewer irregularities than baseline methods. In Sec.~\ref{sec:d_gradfield}, we consider high-dimensional settings with $d\in \{32, 256, 1024\}$ and  demonstrate that our methods achieve an improvement of up to 10 dB when learning the gradient of a convex potential. For nonconvex potentials, the improvement reaches as high as 15 dB.}

\subsubsection{2D Gradient Field}
\label{sec:2d_gradfield}

\begin{figure}[t]
\centering 
\begin{subfigure}{0.4\columnwidth}
  \centering
\includegraphics[width=\linewidth]{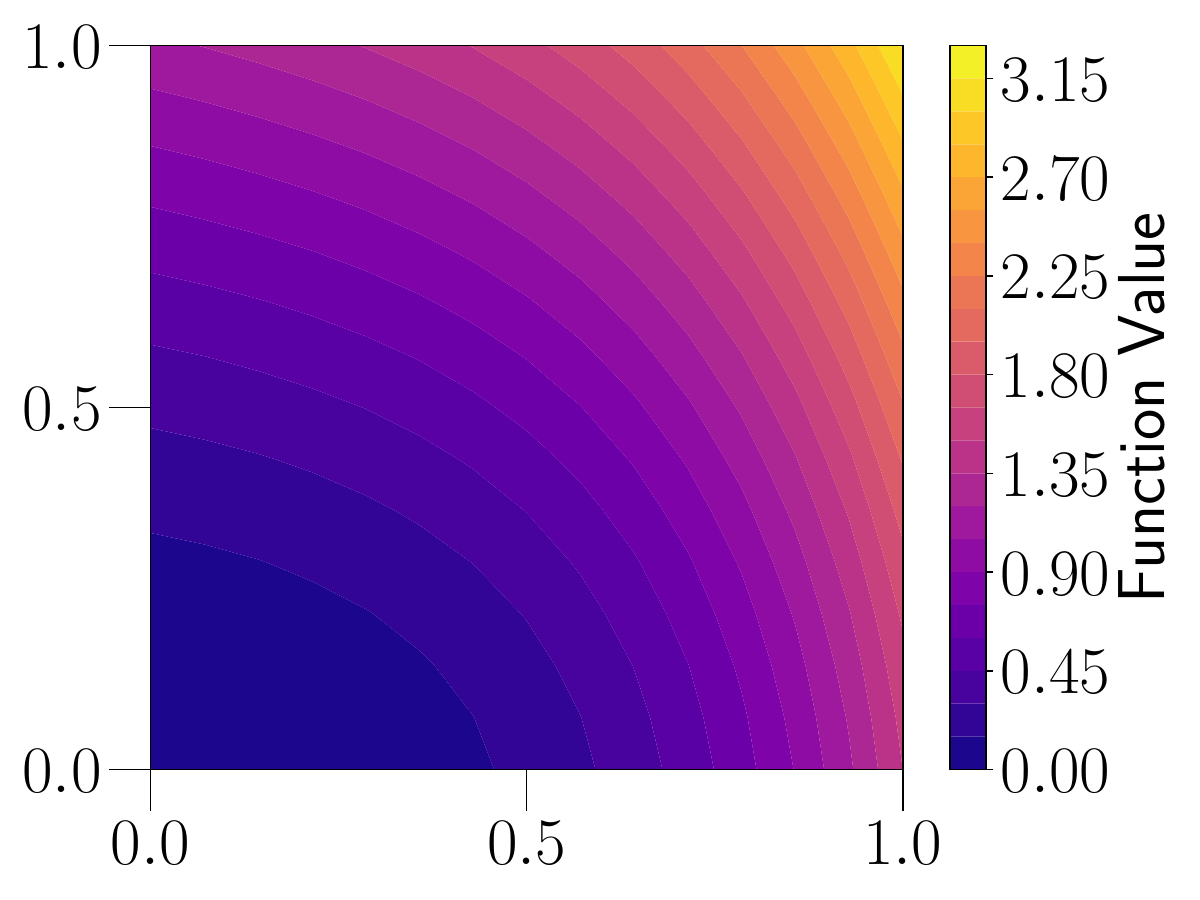}  
  \caption{Contour plot of $F(\x)$}
\end{subfigure}
\begin{subfigure}{0.4\columnwidth}
  \centering
  \includegraphics[width=\linewidth]{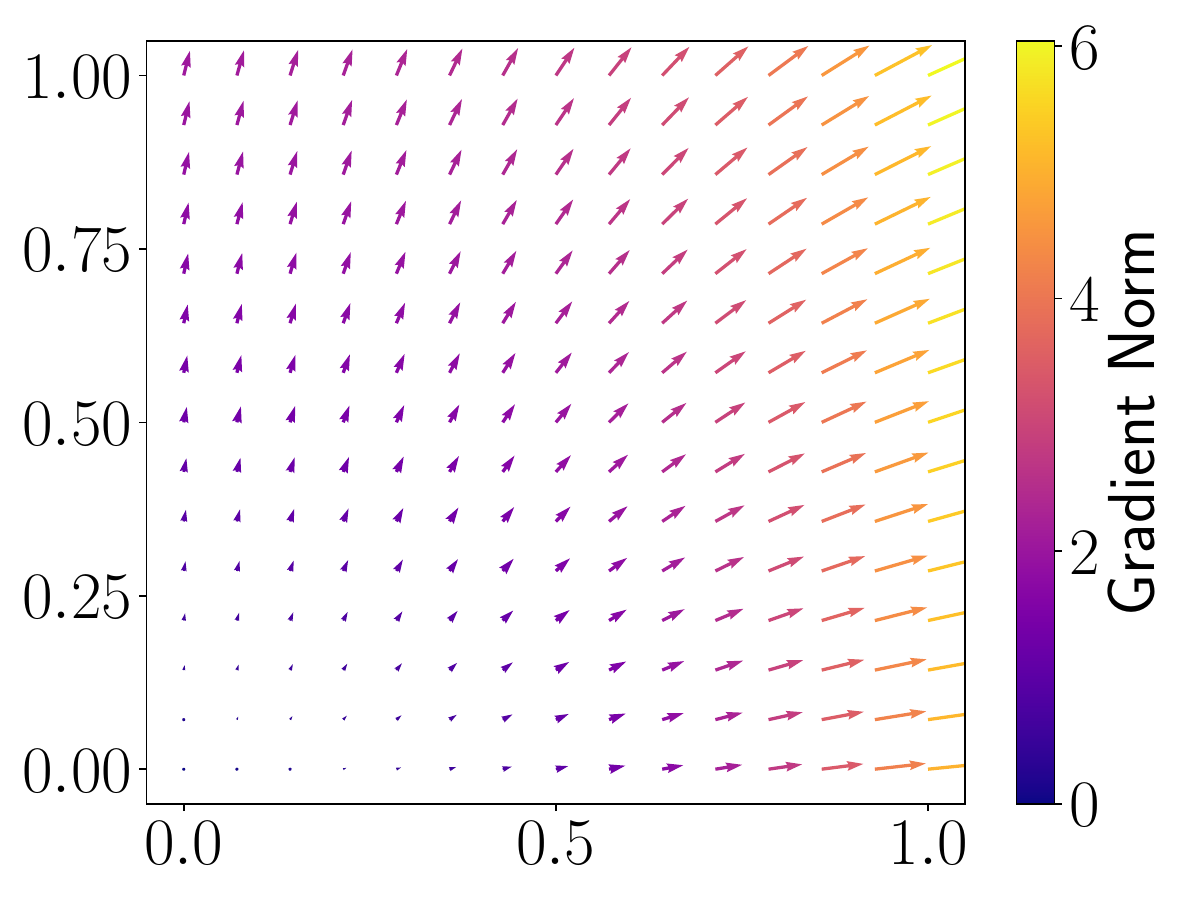} 
  \caption{Quiver plot of $\nabla F(\x)$}
\end{subfigure}

\begin{subfigure}{0.4\columnwidth}
  \centering
  \includegraphics[width=\linewidth]{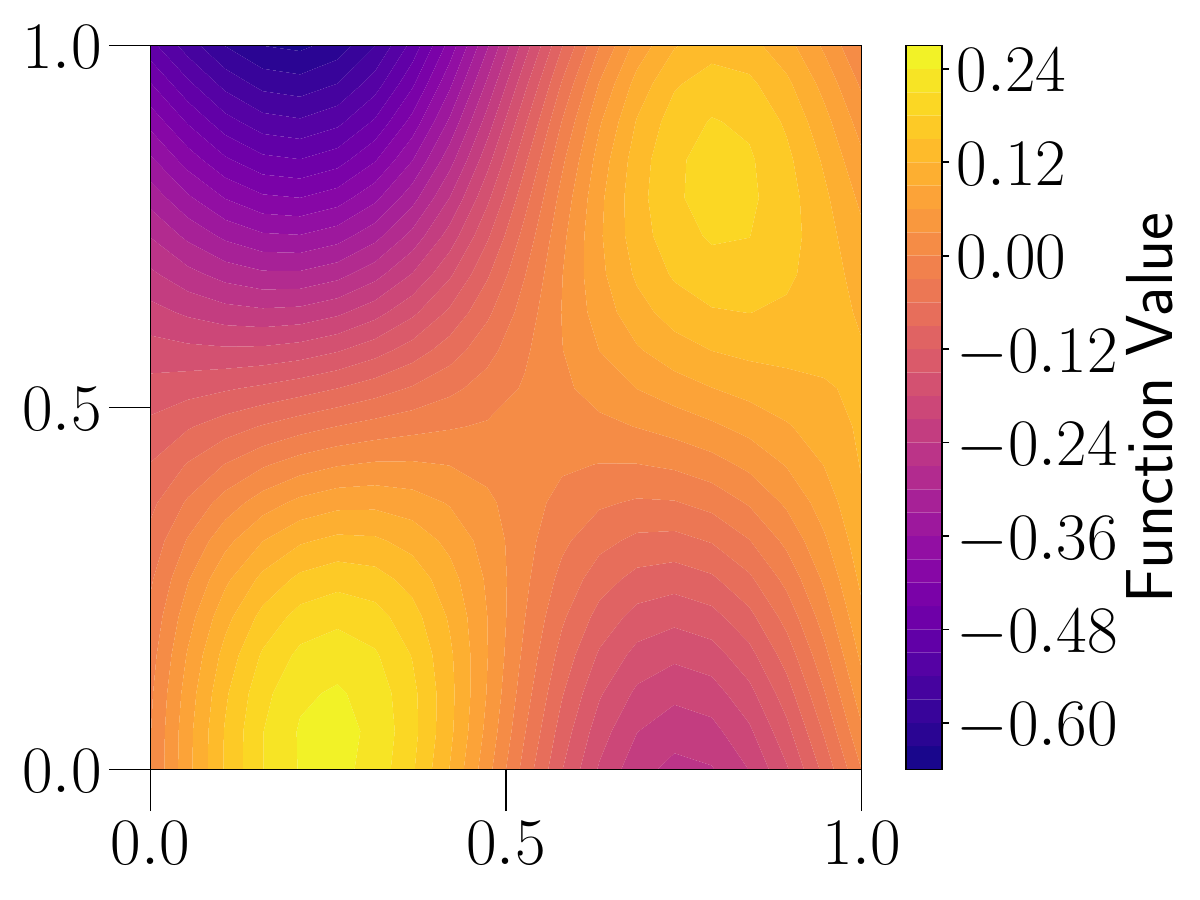}  
  \caption{Contour plot of $G(\x)$}
\end{subfigure}
\begin{subfigure}{0.4\columnwidth}
  \centering
  \includegraphics[width=\linewidth]{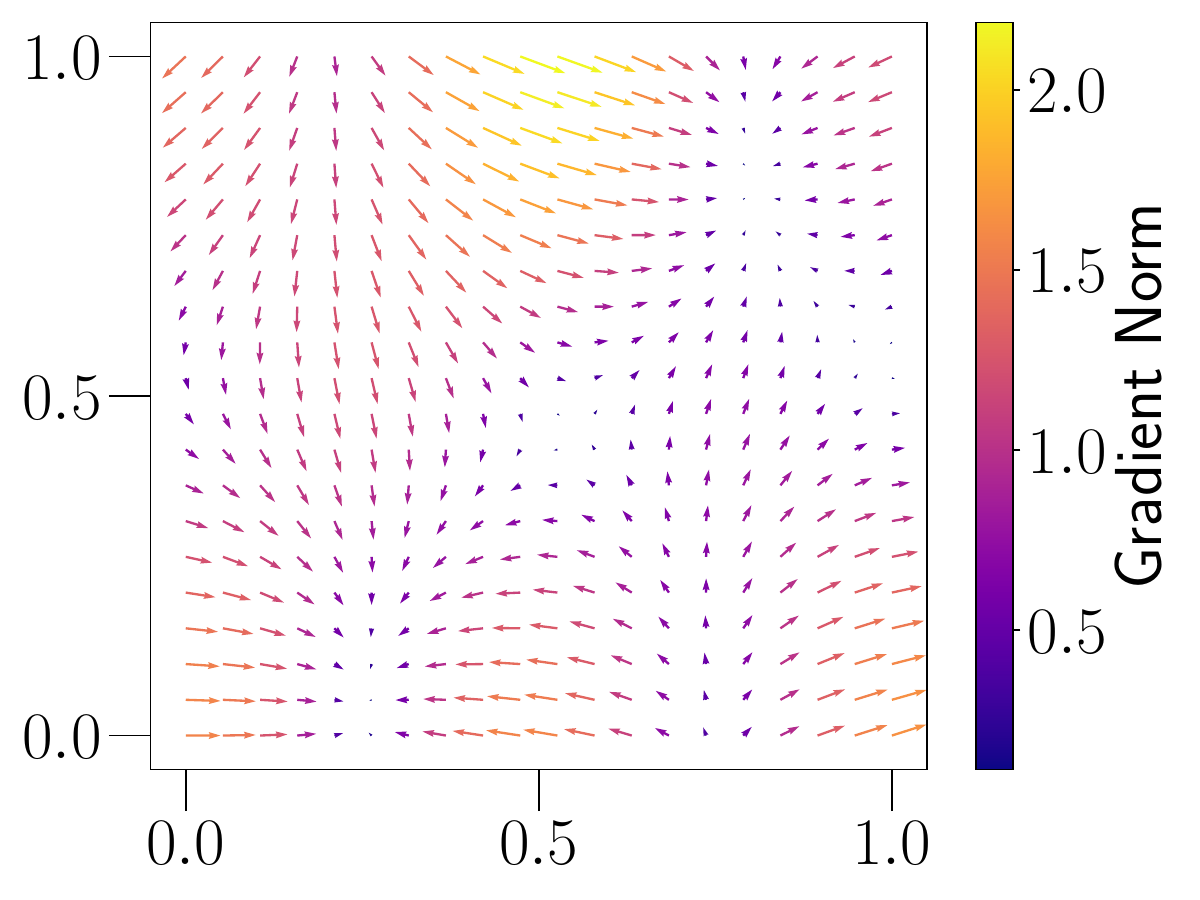} 
  \caption{Quiver plot of $\nabla G(\x)$}
\end{subfigure}
\caption{Test functions $F$ in \eqref{eq:2d_convex_test_fn} and $G$ in \eqref{eq:2d_nonconvex_test_fn}, along with their gradients, for $d=2$ over the unit square.}
\label{fig:2d_test_functions}
\end{figure}

 We start with $d=2$, where the error incurred by each method can be visualized on the unit square.
For the convex test function, we consider learning the gradient of the following benchmark potential function from~\cite{ICGN}, which is convex on the unit square $x_1,x_2 \in [0,1]$:
\begin{align}
    F(x_1, x_2) &= x_1^4 + \frac{x_1^2}{2} + \frac{x_1x_2}{2} + \frac{3x_2^2}{2} - \frac{x_2^3}{3}
    \label{eq:2d_convex_test_fn}
\end{align}
For the nonconvex test function, we consider learning the gradient of the following potential over the unit square:
\begin{align}
    G(x_1,x_2) &= \frac{1}{4}\sin(2\pi x_1)\cos(\pi x_2) + \frac{x_1x_2}{2} - \frac{x_2^2}{2}
    \label{eq:2d_nonconvex_test_fn}
\end{align}
Contour plots of $F$ and $G$, along with their corresponding gradients, are shown in Fig.~\ref{fig:2d_test_functions}.
For the convex test function $F(\x)$ in~\eqref{eq:2d_convex_test_fn}, we compare the proposed \cmgn\ and \mmgn, respectively defined in~\eqref{eqn:cmgn_start}-\eqref{eqn:cmgn_end} and~\eqref{eqn:m-mgn_a}-\eqref{eqn:m-mgn_b}, with the Input Convex Neural Network (ICNN) \cite{ICNN}, Input Convex Gradient Network (ICGN) \cite{ICGN}, and Convex Ridge Regularizer (CRR) \cite{EPFL}. Similarly for the nonconvex test function $G(\x)$ in~\eqref{eq:2d_nonconvex_test_fn}, we compare the proposed modular gradient network (\gnm) and cascaded gradient network (\gnc) to a multilayer perceptron (MLP) and the ridge regularizer (RR) network proposed in \cite{goujon2024learning}. The input to all models is a point $\x \in [0, 1]^2$. All networks, excluding the ICNN and MLP, are trained to directly output $\nabla F(\x)$ or $\nabla G(\x)$. The input to the ICNN and MLP is also $\x$ and the gradient of the networks, with respect to 
$\x$, are trained to approximate $\nabla F(\x)$ or $\nabla G(\x)$, respectively. The specific architecture configurations for each model are provided in Appendix~\ref{sec:appendix_gradfield2d}. All models are constrained to have roughly 100 parameters. We sample 100,000 training points uniformly at random on the unit square and separately sample 10,000 validation points. 
All models are trained using mean squared error (MSE) loss and their performance is evaluated on a uniform grid of points in the unit square. The $\ell_2$ norm of the gradient prediction error is shown for each model in Fig.~\ref{fig:grad_field_results2d}.

\begin{figure*}[htpb]
\centering 
\begin{subfigure}{\linewidth}
  \centering
  \includegraphics[width=\linewidth, trim=0 0 9cm 0, clip]{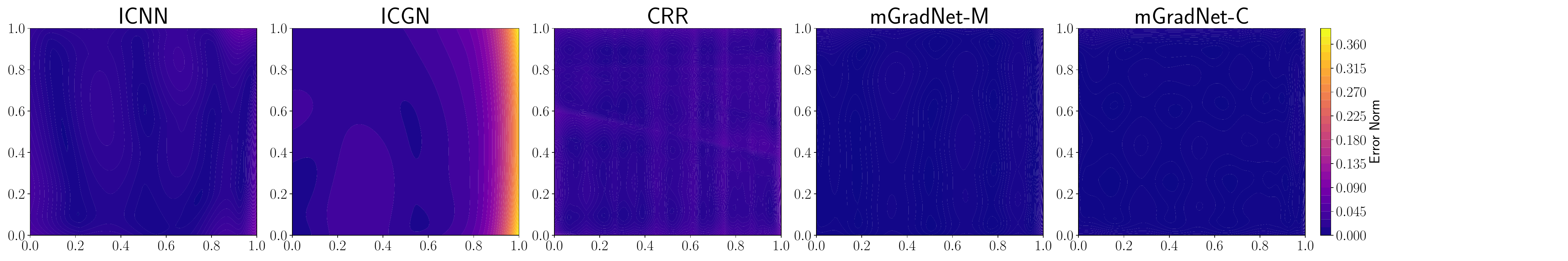}  
  \caption{$\ell_2$ norm between the predicted gradient and $\nabla F(\x)$, the true gradient of $F$ in \eqref{eq:2d_convex_test_fn}}
\end{subfigure}

\begin{subfigure}{\linewidth}
  \centering
  \includegraphics[width=\linewidth, trim=0 0 7.5cm 0, clip]{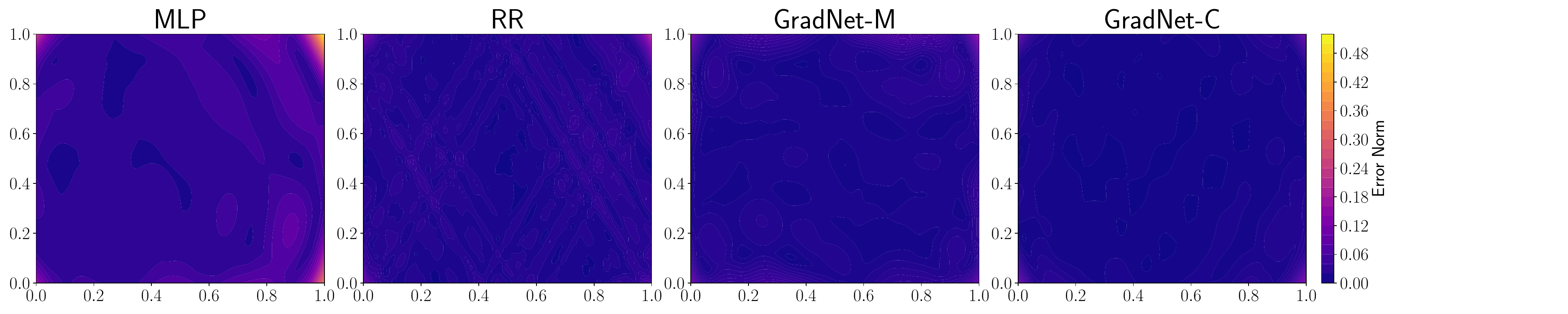}  
  \caption{$\ell_2$ norm between the predicted gradient and $\nabla G(\x)$, the true gradient in \eqref{eq:2d_nonconvex_test_fn}}
\end{subfigure}
\caption{Gradient field learning results for $d=2$}
\label{fig:grad_field_results2d}
\end{figure*}

We observe from Fig.~\ref{fig:grad_field_results2d} that, for learning $\nabla F(\x)$, the proposed \mmgn\ and \cmgn\ exhibit significantly lower error rates than that of the ICNN, ICGN, and CRR. Fig.~\ref{fig:grad_field_results2d} also shows that the \mmgn\ and \cmgn\ are able to effectively learn the gradient at all areas of the unit square, whereas the ICNN, ICGN, and CRR tend to underperform around the edges of square. Similar trends are observed in the nonconvex case where the \gnc, \gnm, and RR outperform the MLP.

\subsubsection{$d$-Dimensional Gradient Field, $d \geq 32$} 
\label{sec:d_gradfield}

We proceed to consider gradient field learning in higher dimensions. {Our main findings are as follows. We observe that our \cmgn\ and \gnc\ architectures consistently outperform the CRR \cite{EPFL} and RR \cite{goujon2024learning}, both of which also use elementwise activations. Our \mmgn\ consistently outperforms all other methods for learning the gradient of a convex potential across all settings considered. Meanwhile, for the nonconvex potential, the best-performing method is consistently either our proposed \gnc\ or the \gnm.}

For the convex test function we use the following positive definite matrices defined in~\cite{svanberg2002class}:
\begin{align}
   \bs{S}_{ij} &= \frac{2 + \sin(4\pi\alpha_{ij})}{(1 + |i-j|\ln d)}\qquad   \bs{P}_{ij} = \frac{1 + 2\alpha_{ij}}{(1 + |i-j|\ln d)}\\
   \bs{Q}_{ij} &= \frac{3 - 2\alpha_{ij}}{(1 + |i-j|\ln d)}\qquad
   \alpha_{ij} = \frac{i + j - 2}{2d - 2} \in [0, 1]\;\forall i,j
\end{align}
where $1 \leq {i,j} \leq d$ and $\bs{S}_{ij}$ denotes entry ${ij}$ of the matrix $\bs{S}$. We consider learning the gradient (where it exists) of the convex, piecewise quadratic function:
\begin{align}
    \bs{z}_i &= \bs{x}_i - 0.5\\
    \widetilde{F}(\x) &= \max \{\bs{z}^\top \bs{S} \bs{z}, \bs{z}^\top \bs{P} \bs{z}, \bs{z}^\top \bs{Q} \bs{z}\} 
    \label{eq:convex_fn}
\end{align}
over the unit hypercube $\x \in [0, 1]^d$ for $d\in \{32, 256, 1024\}$. We also consider a nonconvex setting in which we train the appropriate models to learn the score function of a Gaussian mixture model (GMM), $\nabla_{\x} \ln \widetilde{G}(\x)$, where $\ln \widetilde{G}(\x)$ is the log probability density function:
\begin{align}
    \ln \widetilde{G}(\x) &= \ln \left(\sum_{i=1}^N \alpha_i N(\x; \bs{\mu}_i, \bs{\Sigma}_i)\right)
    \label{eq:nonconvex_fn}
\end{align}
In~\eqref{eq:nonconvex_fn} above, $N(\x; \bs{\mu}_i, \bs{\Sigma}_i)$ denotes the probability density function of a multivariate normal distribution with mean $\bs{\mu}_i \in \R^d$ and $d \times d$ covariance matrix $\bs{\Sigma}_i$. Since most of the probability mass of a $d$-dimensional standard normal distribution lies in a thin annulus that is centered at the origin and has inner and outer radii close to $\sqrt{d}$ \cite{blum2020foundations}, we ensure interesting score functions in the unit hypercube by drawing each component of $\bs{\mu}_i$ uniformly at random from $[0.3, 0.7]$ and setting $\bs{\Sigma} = 2\sqrt{d}\bs{I}$. We select equal weights for the Gaussians. We train all models in these experiments for 10,000 iterations, where at each iteration we randomly sample 100 points from $[0,1]^d$ to update the model parameters. The specific model configurations and additional training details are provided in Appendix \ref{sec:appendix_gradfield}.

The results for learning the gradients of~\eqref{eq:convex_fn} and~\eqref{eq:nonconvex_fn} are shown in Table~\ref{table:convex_gradfield} and Table~\ref{table:nonconvex_gradfield} respectively.

We observe in  Table~\ref{table:convex_gradfield} that the \mmgn\ outperforms the baseline methods by up to 10 dB in the high-dimensional setting of $d=1024$. {Furthermore, although they parameterize the same space of functions, we observe that the \cmgn\ outperforms the CRR from \cite{EPFL} in the settings $d=256$ and $d=1024$ by margins of roughly 4 dB and 2 dB respectively}. 

\begin{table}[ht]
\centering
\caption{Learning the Gradient of Convex Potential \eqref{eq:convex_fn} (15 trials). *Our Methods}
\begin{tabular}{lr@{ $\pm$ }lr@{ $\pm$ }lr@{ $\pm$ }l}
\toprule
& \multicolumn{2}{c}{$d=32$} & \multicolumn{2}{c}{$d=256$} & \multicolumn{2}{c}{$d=1024$} \\ 
Model & \multicolumn{2}{c}{MSE (dB)} & \multicolumn{2}{c}{MSE (dB)} & \multicolumn{2}{c}{MSE (dB)}\\ \midrule
ICNN \cite{ICNN} & -12.76 & 0.03 &  -9.83 & 0.01 & -0.36 & 0.01\\
CRR \cite{EPFL} & -13.86 & 0.03   &  -6.24 & 0.04  & -6.53 & 0.01        \\
{\textbf{\cmgn*}} & -12.80 & 0.03  &  -10.64 & 0.01  & -8.75 & 0.02     \\
{\textbf{\mmgn*}} & \textbf{-15.34} & 0.03  &  \textbf{-11.30} & 0.02 & \textbf{-10.42} & 0.01  \\
\bottomrule
\end{tabular}
\label{table:convex_gradfield}
\end{table}

Similarly for the the score learning task, we observe in Table~\ref{table:nonconvex_gradfield} that the \gnm\ and \gnc\ significantly outperform all other methods. The \gnc\ outperforms the RR by nearly 8 dB for $d=32$ and by over 4 dB for $d=256$. The \gnm\ outperforms the RR by a margin of over 5 dB for $d = 1024$, thus demonstrating its effectiveness in high dimensions. The MLP baseline underperformed all other considered methods, especially for $d=1024$, where its performance lagged by nearly 9 dB behind the second worst model, namely the RR. In both the convex and nonconvex experiments, the modular architectures, \mmgn\ and \gnm, perform better than the other methods in the high-dimensional setting $d=1024$. 

\begin{table}[ht]
\centering
\caption{Learning the Nonconvex Score Function Corresponding to \eqref{eq:nonconvex_fn} (15 trials). *Our Methods}
\begin{tabular}{lr@{ $\pm$ }lr@{ $\pm$ }lr@{ $\pm$ }l}
\toprule
      & \multicolumn{2}{c}{$d=32$} & \multicolumn{2}{c}{$d=256$} & \multicolumn{2}{c}{$d=1024$} \\ 
Model & \multicolumn{2}{c}{MSE (dB)} & \multicolumn{2}{c}{MSE (dB)} & \multicolumn{2}{c}{MSE (dB)}\\ \midrule
MLP   & -33.20 & 0.04 & -18.21 & 0.14  & -2.78 & 0.01 \\
RR \cite{goujon2024learning}   & -32.50  & 0.30 & -20.30 & 2.76 & -11.26 & 0.89    \\
{\textbf{\gnc*}}  & \textbf{-44.43}  & 0.07 & \textbf{-24.90} & 0.07  & -12.84 & 0.04   \\
{\textbf{\gnm*}}  & -37.66  & 0.02 & -21.72 & 0.14  & \textbf{-17.02} & 0.03           \\
\bottomrule
\end{tabular}
\label{table:nonconvex_gradfield}
\end{table}


\subsection{Hamiltonian Gradients for Two-Body Dynamics}

{In this section, we train \gns\ to learn gradients of the Hamiltonian function for the two-body problem and subsequently use the learned models to predict the trajectories of the two objects. We observe that our \gnm\ outperforms the baseline Hamiltonian NN~\cite{HamiltonianNN} by over 3 dB and the RR \cite{goujon2024learning} by over 11 dB. As in the gradient field experiments, our \gnc\ architecture again outperforms the RR \cite{goujon2024learning} -- this time by over 3 dB. The performance gaps that arise in this experiment directly validate the theoretical differences in approximation capabilities between the baseline methods and \gns\ (see Sec.~\ref{sec:universal_approximation}).}

Hamiltonian mechanics has many fundamental applications ranging from classical mechanics and thermodynamics to quantum physics. In fact, predicting the dynamics of objects in a system can be achieved by learning the gradient of the corresponding scalar-valued Hamiltonian function.
We define $\q$ and $\p$ respectively as the vectorized position and momentum information of all objects in the system. The Hamiltonian $\mathcal{H}(\q,\p)$ characterizes the total energy of the system and offers a complete description of the system. The time derivatives of the position $\q$ and momentum $\p$ (which are respectively the velocities and forces) of all objects can be obtained by differentiating the Hamiltonian of the system: \begin{equation}
    \frac{d\q}{dt} = \frac{\partial\mathcal{H}}{\partial\p} \qquad -\frac{d\p}{dt} = \frac{\partial\mathcal{H}}{\partial\q}
    \label{eqn:general_hamiltonian}
\end{equation}
Hence, the Hamiltonian formulation of the system dynamics provides a simple method for obtaining position, momentum, velocity, and force information. Following the procedure in~\cite{HamiltonianNN}, we train models to take $\q,\p$ as input and output $d\q/dt$ and $-d\p/dt$. By \eqref{eqn:general_hamiltonian}, training these models corresponds to learning gradients of an unknown Hamiltonian function -- a direct and quintessential application for \gns.

The system for the two-body problem is composed of two point particles interacting through an attractive force like gravity. It has the following nontrivial, nonconvex Hamiltonian:
\begin{equation}
    \mathcal{H}(\q,\p) = \frac{\|\p_{CM}\|_2^2}{m_1+m_2} + \frac{\|\p_1\|_2^2+\|\p_2\|_2^2}{2\mu}+\frac{gm_1m_2}{\|\q_1-\q_2\|_2^2}
    \label{eqn:Hamiltonian_2body}
\end{equation}
where $\mu$ is the reduced mass and $\p_{CM}$ is the
momentum of the center of mass. Using the experimental setup and training procedure\footnote{Our results are obtained by evaluating all models trained using the hyperparameter search detailed in Appendix~\ref{sec:appendix_hamiltonian}.} in \cite{HamiltonianNN}, we consider a two-dimensional physical ambient space that induces eight degrees of freedom for the system: two-dimensional position and momentum for each object. This means $\q,\p\in\R^8$. We compare our networks with the Hamiltonian neural network (NN) and multi-layer perceptron baseline from \cite{HamiltonianNN}. Unlike the \gns\ introduced in this paper, the Hamiltonian NN does not offer theoretical guarantees concerning universal approximation of gradient functions.

\begin{figure*}[h]
\centering
\begin{subfigure}[t]{.3\linewidth}  \includegraphics[width=\linewidth]{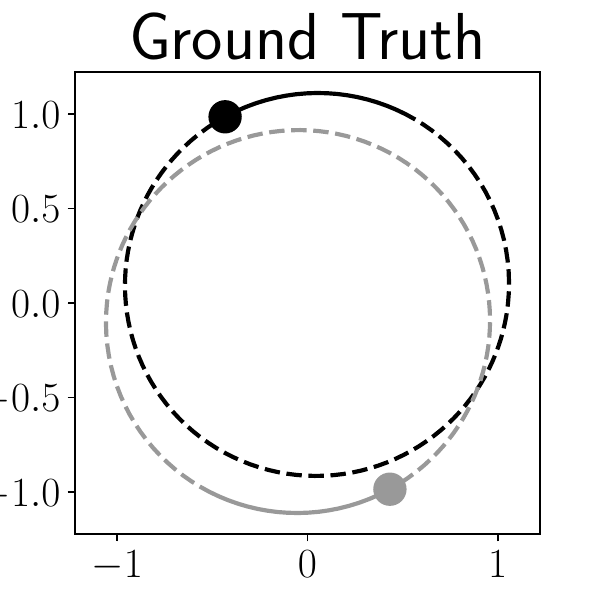}
\end{subfigure}
\begin{subfigure}[t]{.3\linewidth}
\includegraphics[width=\linewidth]{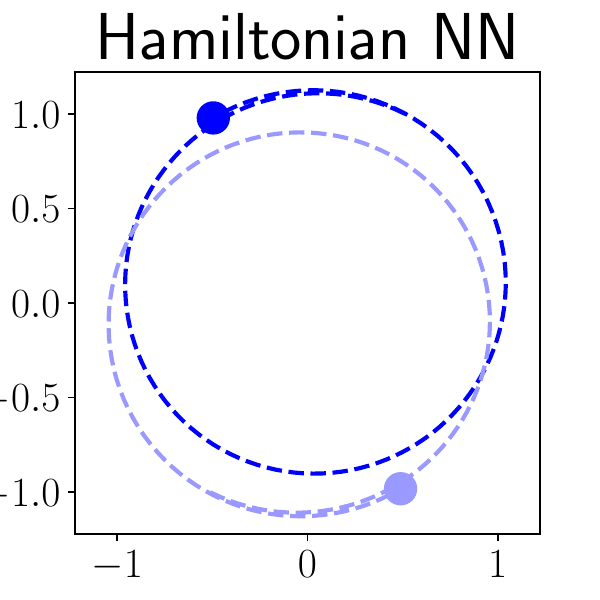}
\end{subfigure}
\begin{subfigure}[t]{.3\linewidth}
\includegraphics[width=\linewidth]{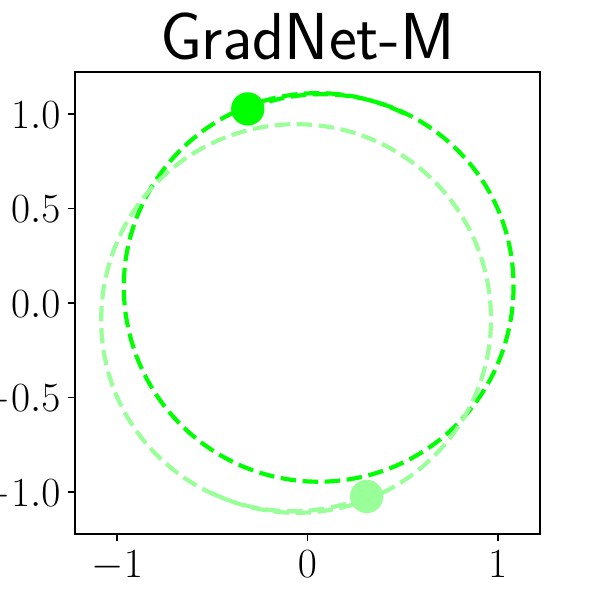}
\end{subfigure}

\begin{subfigure}[t]{.3\linewidth}
\includegraphics[width=\linewidth]{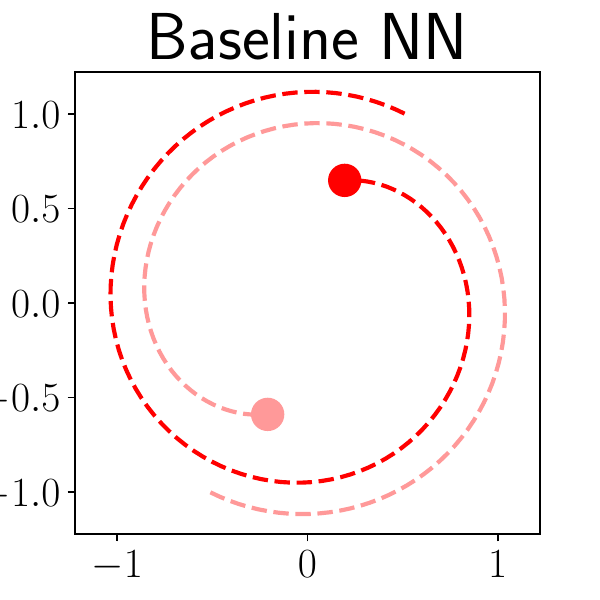}
\end{subfigure}
\begin{subfigure}[t]{.3\linewidth}
\includegraphics[width=\linewidth]{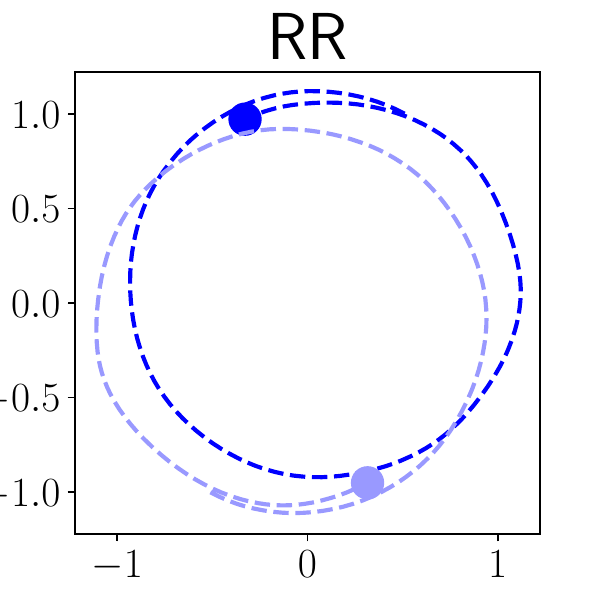}
\end{subfigure}
\begin{subfigure}[t]{.3\linewidth}
\includegraphics[width=\linewidth]{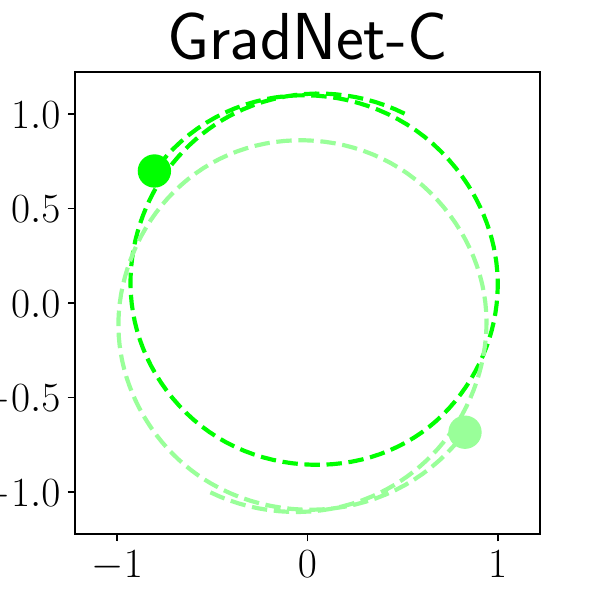}
\end{subfigure}
\caption{2-body problem object trajectories from unrolling dynamics using models that learn gradients of the Hamiltonian. Our proposed methods are in green.}
\label{fig:hamiltonian_2body}
\end{figure*}

\begin{table}[h]
\centering
\caption{Object Coordinate MSE and Total System Energy MSE of 2-Body Problem Trajectories Unrolled using \eqref{eqn:general_hamiltonian} and Hamiltonian Gradients Estimated by each Model (15 trials). *Our Methods
}

\begin{tabular}{lr@{ $\pm$ }lr@{ $\pm$ }l}
\toprule 
Model & \multicolumn{2}{l}{Coordinate MSE (dB)} & \multicolumn{2}{l}{Energy MSE (dB)} \\ \midrule
Baseline NN & -4.90	& 2.49	& -9.06 & 1.79\\
Hamiltonian NN \cite{HamiltonianNN} & -20.01	& 1.28 & -51.43 & 0.68 \\
RR \cite{EPFL} & -12.77 &	1.74 &	-19.53	& 2.82 \\
\textbf{\gnc*} & -16.32	& 1.70 &-44.29 &	0.55\\
\textbf{\gnm*} & \textbf{-23.84} & 1.61 & \textbf{-52.31} & 0.59 \\
\bottomrule
\end{tabular}

\label{tab:hamiltonian_2body_results}
\end{table}}

In Fig.~\ref{fig:hamiltonian_2body}, we plot the ground truth trajectories of the two bodies along with the trajectories unrolled using the estimated gradients of the Hamiltonian in \eqref{eqn:Hamiltonian_2body}. Table~\ref{tab:hamiltonian_2body_results} shows the corresponding coordinate MSEs for the unrolled trajectories and reflects the results in Fig.~\ref{fig:hamiltonian_2body}. Since the energy in the system should remain constant, the energy MSEs in Table~\ref{tab:hamiltonian_2body_results} reflect how effectively the Hamiltonian gradient predictions of each method conserve the total energy of the system. We see that the \gnm\ outperforms the Hamiltonian NN baseline by over 3 dB in terms of coordinate MSE and by approximately 1 dB in terms of energy MSE. Similarly, the \gnm\ outperforms the RR by 11 dB for coordinate MSE and 32 dB for energy MSE. We attribute this gain to the universal approximation capabilities of of the \gnm\ discussed in Sec.~\ref{sec:universal_approximation}. The \gnc\ outperforms the RR and demonstrates better approximation capabilities for the gradient of~\eqref{eqn:Hamiltonian_2body}, which is not a sum of ridge functions. 


\section{Conclusion}
In this work, we presented gradient networks (\gns): novel classes of neural networks for directly parameterizing and learning gradients of functions. We first introduced gradient networks that are guaranteed to be gradients of scalar-valued functions. We subsequently demonstrated methods to convert gradient networks into monotone gradient networks (\mgns) that are guaranteed to correspond to the gradients of convex functions. Notably, we proposed a general framework for designing (monotone) gradient networks and rigorously established their universal approximation capabilities for several key function classes. Our experimental results on gradient field and dynamics learning problems showcase the practical utility of \gns\ and reinforce our theoretical findings. We observed improvements of up to 15 dB over existing methods for approximating gradient fields and achieved improvements of up to 11 dB for predicting the dynamics of the two-body problem.

\newpage
\section*{Appendix}
\appendix
All models are trained using the Adam optimizer with standard momentum parameters 0.9 and 0.999.

\section{Gradient Field Experiment Details: $d=2$}
\label{sec:appendix_gradfield2d}
The \mmgn\ has 4 modules, each with hidden dimension 7, $\rho(x) = 1$, and $\sigma(\x)$ as the softmax function. The module outputs are conically combined via learnable, nonnegative parameters.  The \gnm\ is identical to the described \mmgn, however the module outputs are combined linearly rather than conically. The \cmgn\ has $L=3$ hidden layers, each with a hidden dimension of 7, elementwise tanh activation, and intermediate nonnegative weight vectors $\balpha_\ell, \bbeta_\ell$. The \gnc\ is identical to the mentioned \cmgn\ except $\balpha, \bbeta$ are unconstrained. We compare to a 2 hidden layer ICNN with hidden dimension 7 and softplus activations. The ICGN has the form $\sigma(\W\x + \bs{b})$ with hidden dimension 32 and sigmoid activation. The CRR has hidden dimension 10 and each learnable spline activation has 7 knots. The RR has the same architecture as the CRR except the spline activation functions are not constrained to be nondecreasing. The MLP has 3 hidden layers and  hidden dimension 6. All models are trained for 200 epochs with batch size 1000 and learning rate 0.005.

\section{Gradient Field Experiment Details: $d\geq 32$}
\label{sec:appendix_gradfield}
All models have the same architecture as described in Appendix~\ref{sec:appendix_gradfield2d} with some minor modifications. The CRR and RR each have 41 knots. The \cmgn\ has intermediate, learnable activations of the form $\alpha \tanh(x) + \beta(x - \tanh(x))$ where $\alpha, \beta$ are constrained to be nonnegative. The \gnc\ uses the same activation, but with unconstrained $\alpha, \beta$. The \mmgn\ has intermediate, learnable activation $\alpha \softmax(x) - \beta \mathrm{softmin}(\x)$ where $\alpha, \beta$ are constrained to be nonnegative. The \gnm\ uses the same activation with unconstrained $\alpha, \beta$. The hidden dimension of each model is scaled such that each model has $1024\cdot d$ parameters. All models are trained for 10,000 iterations with randomly sampled batches of size 1,000 from $[0, 1]^d$. We test learning rates of 0.01, 0.001, and 0.0001 for each model, and report results corresponding to the learning rate that achieves the best performance. For each test trial, we evaluate the trained models on 10,000 points randomly sampled from $[0, 1]^d$.

\section{Hamiltonian Experiment Details}
\label{sec:appendix_hamiltonian}
We use the baseline MLP and Hamiltonian NN architectures described in \cite{HamiltonianNN} and tune over the hyperparameters listed in the appendix of \cite{HamiltonianNN}. To closely match the number of parameters used in the best-performing Hamiltonian NN with hidden dimension $100$, we use a \gnm\ with 4 modules and hidden dimension 256, with $\rho(x) = 1$, and $\sigma(\x)$ as the softmax function. The outputs of the modules are linearly combined via arbitrary learnable parameters. The \gnc\ has $L=4$ hidden layers, each with a hidden dimension of 256, elementwise tanh activation, and intermediate weight vectors $\balpha_\ell, \bbeta_\ell$. The RR has a hidden dimension of 220 and each learnable spline activation has 41 knots. We tuned the RR's range parameter over 0.01, 0.1, 1, and 10, and sparsity parameter over 0, 1e-1, 1e-2, 1e-4, 1e-6, and 1e-8.
Like \cite{HamiltonianNN}, we use a batch size of 200 and train for 10,000 steps. We tuned over learning rates 1e-1, 1e-2, 1e-3, 1e-4 and 1e-5 for each model, and report results corresponding to the best-performing model. For each trial at test time, we unroll from $t=0$ to $t=60$, where $\Delta t = 0.03$ (2000 steps).

\newpage
\bibliographystyle{IEEEtran}
\bibliography{refs}

\newpage
\begin{center}
    \includegraphics[width=1in,height=1.25in,clip,keepaspectratio]{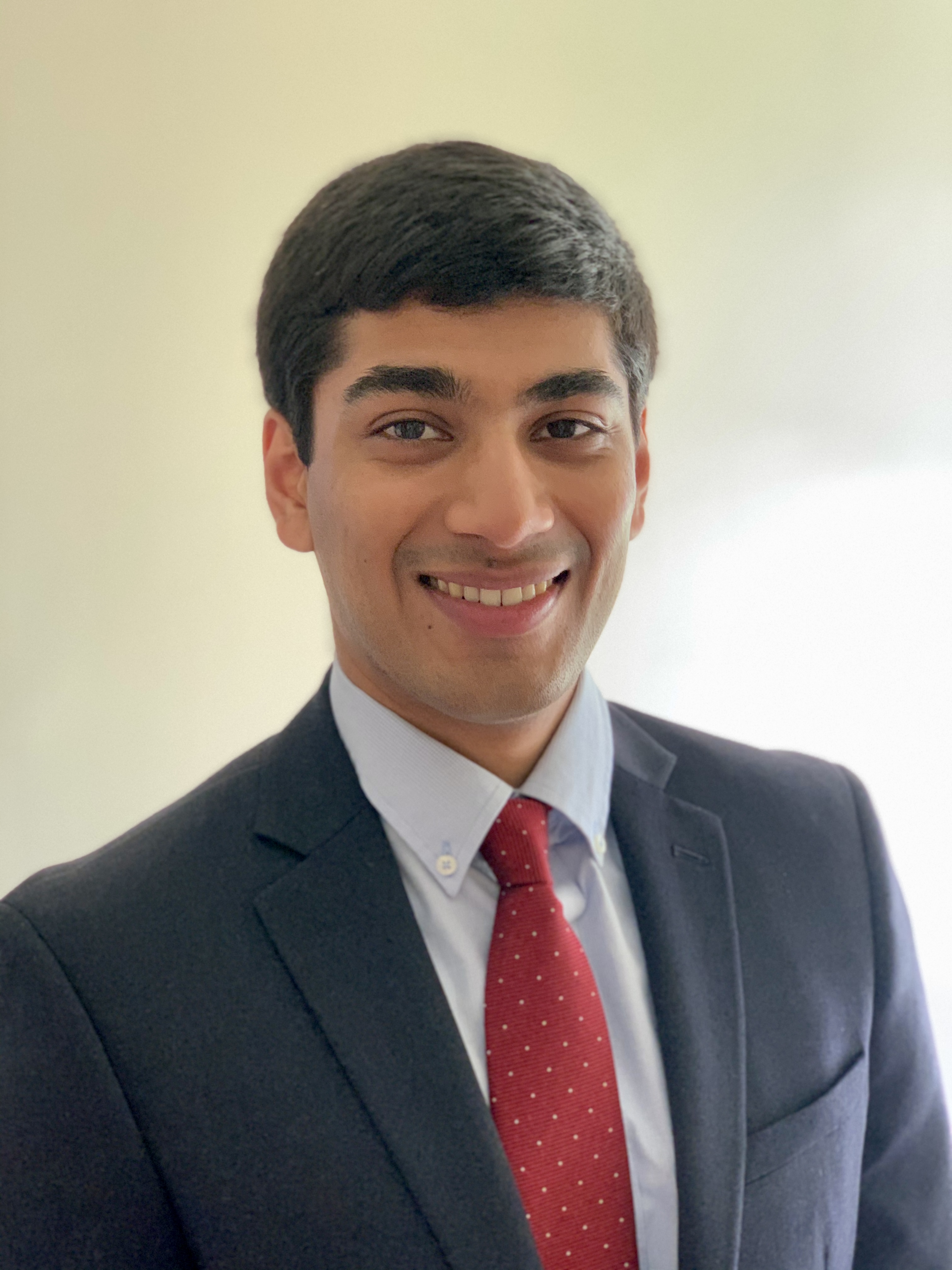}
\end{center} 
\textbf{Shreyas Chaudhari}
(Student Member, IEEE) received his B.S. in Electrical and Computer Engineering in 2018 from The Ohio State University, Columbus, OH, USA. He earned his M.S. in Electrical and Computer Engineering from Carnegie Mellon University, Pittsburgh, PA, USA in 2021, where he is currently pursuing a Ph.D. in the same discipline. In 2021, he was awarded the National Science Foundation Graduate Research Fellowship. His research interests include the foundational aspects of deep learning as well as applications of machine learning in signal processing.

\begin{center}
    \includegraphics[width=1in,height=1.25in,clip,keepaspectratio]{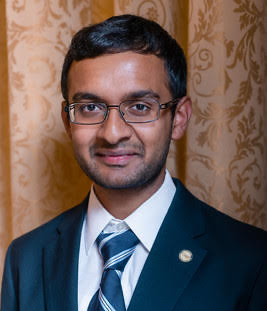}
\end{center}

\textbf{Srinivasa Pranav}
(Student Member, IEEE) received a B.S. degree in Electrical Engineering and Computer Sciences, a B.A. degree in Applied Mathematics, and a B.S. degree in Business Administration in 2019 from the University of California at Berkeley, Berkeley, CA, USA. He received an M.S. in Electrical and Computer Engineering in 2020 from Carnegie Mellon University, Pittsburgh, PA, USA. He is currently a Ph.D. student in Electrical and Computer Engineering at Carnegie Mellon University and is a National Science Foundation Graduate Research Fellow and ARCS Scholar. His research interests include machine learning, distributed optimization, and peer-to-peer deep learning.

\begin{center}
    \includegraphics[width=1in,height=1.25in,clip,keepaspectratio]{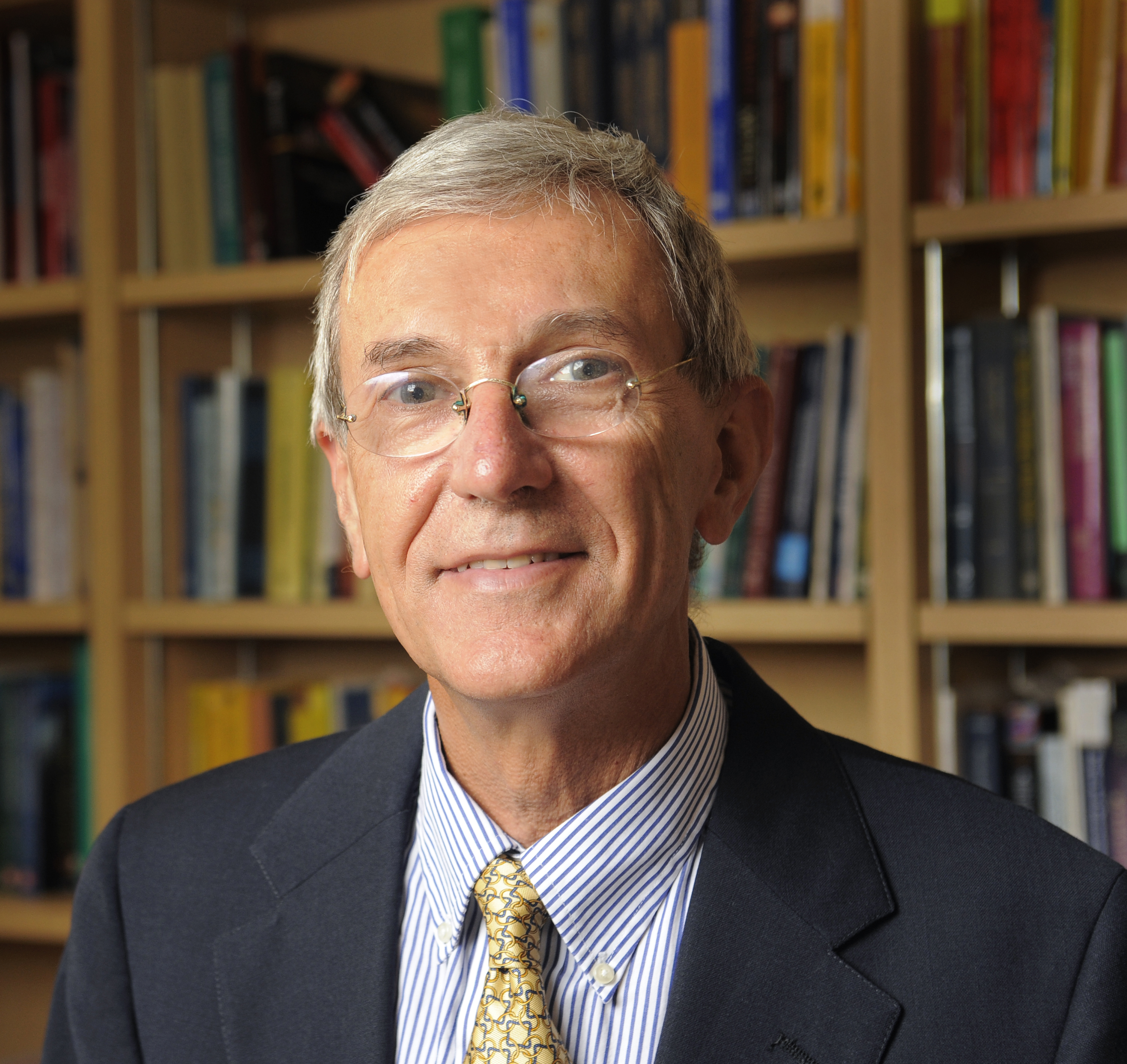}
\end{center}

\textbf{Jos{\'e} M.~F.~Moura}
(Life Fellow, IEEE) is the Philip L.~and Marsha Dowd University Professor at Carnegie Mellon University. His research is in statistical, algebraic, and graph signal processing. The technology of two of his 19 patents (co-inventor Alek Kavcic) is found in the read channel of over 4 billion  hard disk drives (60~\% of all computers sold worldwide since 2003) and the subject of a 2016 \$~750 million settlement between Carnegie Mellon University and a major chip manufacturer. He was Editor in Chief of the IEEE Transactions on Signal Processing, 2008-09 President of the IEEE Signal Processing Society, and 2019 IEEE President and CEO. He is a Fellow of the IEEE, AAAS, and the US National Academy of Inventors, and a member of the Academy of Sciences, Portugal and of the US National Academy of Engineering. He was awarded doctor honoris causa by the University of Strathclyde (Scotland, UK) and by Universidade de Lisboa (Portugal), the Great Cross of Prince Henry from the President of the Republic of Portugal, and the IEEE Jack S. Kilby Signal Processing Medal ``[f]or contributions to theory and practice of statistical, graph, and distributed signal processing.'' He was recognized with the IEEE Haraden Pratt Award.

\end{document}